%% file: Arxiv-Triang.tex
\documentclass[letterpaper, 10 pt, conference]{ieeeconf}  

\usepackage[pdftex]{graphicx}        
\usepackage{textcomp}
\usepackage{algorithm}
\usepackage{algpseudocode}
\usepackage{mathrsfs}
\usepackage{paralist}
\usepackage[font=small,format=plain,labelfont=bf]{caption}
\usepackage{amssymb,amsmath}
\usepackage{footnote}
\usepackage{subfig}
\usepackage{multirow}
\usepackage{array}
\usepackage{float}
\usepackage[usenames,dvipsnames]{color}
\usepackage{complexity}
\usepackage{amsmath} 
\usepackage{bm}
\usepackage{mdwlist}
\usepackage{caption}
\usepackage{cite}

\newcommand{\todo}[1]{
  \textcolor{red}{\footnotesize \textsf{#1}}
}
\renewcommand{\todo}[1]{} 

\newcommand{\figheight}{1in}

\newcommand{\figDim}{1.0}

\newcommand{\proc}{\textsc}

\newcommand{\tablelength}{3.6cm}

\newtheorem{theorem}{Theorem}[section]
\newtheorem{lemma}[theorem]{Lemma}
\newtheorem{definition}{Definition}[section]

\IEEEoverridecommandlockouts                              

\title{\LARGE \bf
Exploration via Structured Triangulation by a Multi-Robot System with Bearing-Only Low-Resolution Sensors
}

\author{SeoungKyou~Lee, 
        Aaron~Becker, 
        S\'andor~P.~Fekete, 
        Alexander~Kr\"oller, and
        ~James~McLurkin 
\thanks{S. Lee, A. Becker, and J. McLurkin are with the
Computer Science Department, Rice University, Houston, TX, 77005 USA e-mail: \tt\small{sl28@rice.edu}.}
\thanks{A. Kr\"oller and S. Fekete are with the Computer Science Department, TU Braunschweig, Braunschweig, Germany.}
}

\begin{document}

\maketitle
\thispagestyle{empty}
\pagestyle{empty}

\begin{abstract}
This paper presents a distributed approach for exploring and triangulating
an unknown region using a multi-robot system.  The objective is to produce a covering of an unknown workspace by a fixed number of robots such that the
covered region is maximized, solving the \emph{Maximum Area Triangulation
Problem} (MATP).  The resulting triangulation is a \emph{physical data
structure} that is a compact representation of the workspace; it contains
distributed knowledge of each triangle, adjacent triangles, and the dual graph of the workspace.  Algorithms can store information in this physical data structure, such as a routing table for robot navigation

Our algorithm builds a triangulation in a closed environment, starting from a single location.  It provides coverage with a breadth-first search pattern and completeness guarantees.  We show the computational and communication requirements to build and maintain the triangulation and its dual graph are small. Finally, we present a physical navigation algorithm that uses the dual graph, and show that the resulting path lengths are within a constant factor of the shortest-path Euclidean distance.  
We validate our theoretical results with experiments on triangulating a region with a system of low-cost robots.
Analysis of the resulting quality of the triangulation shows that most of the triangles are of high quality, and cover a large area.  Implementation of the triangulation, dual graph, and navigation all use communication messages of fixed size, and are a practical solution for large populations of low-cost robots.
%
\end{abstract}

\input{01-Introduction.tex}
\input{02-ModelAndPreliminaries.tex}
\input{03-MATAlgorithm.tex}

\input{04-VirtualDualGraph.tex}
\input{05-DualGraphNavigation.tex}

\input{05-Implementaion.tex}

\input{05-Experiments.tex}
\input{06-Conclusion.tex}

%
%


\bibliographystyle{ieeetran}
\bibliography{lit,bibliography,mclurkin-bibliography}

\end{document}

%% file: 01-Introduction.tex
\section{Introduction and Related Work}
\label{sec:Introduction}

Many practical applications of multi-robot systems, such as search-and-rescue, exploration, mapping and surveillance require robots to disperse across a
large geographic area.  Large populations of robots offer two large advantages: they can search the environment rapidly using a breadth-first approach, and can maintain coverage of the environment after the dispersion is complete.

In this paper, we demonstrate that triangulating the workspace with a
multi-robot system is a useful approach to dispersion and
monitoring. Triangulations are used in a large variety of applications because of their useful properties.  In our application they provide complete coverage, they can be built with only basic local geometry, and they allow
proofs of properties for coverage, navigation, and distributed data storage.
The underlying topological structure of a triangulation allows us to exploit
its dual graph for mapping and routing, with performance guarantees for these
purposes. Fig.~\ref{fig:NavTestResult} shows an example output demonstrating a triangulated network, its dual graph, and a navigating robot.

\begin{figure}[t]
\renewcommand{\figheight}{1.6in}
\centering
\includegraphics[width=\linewidth]{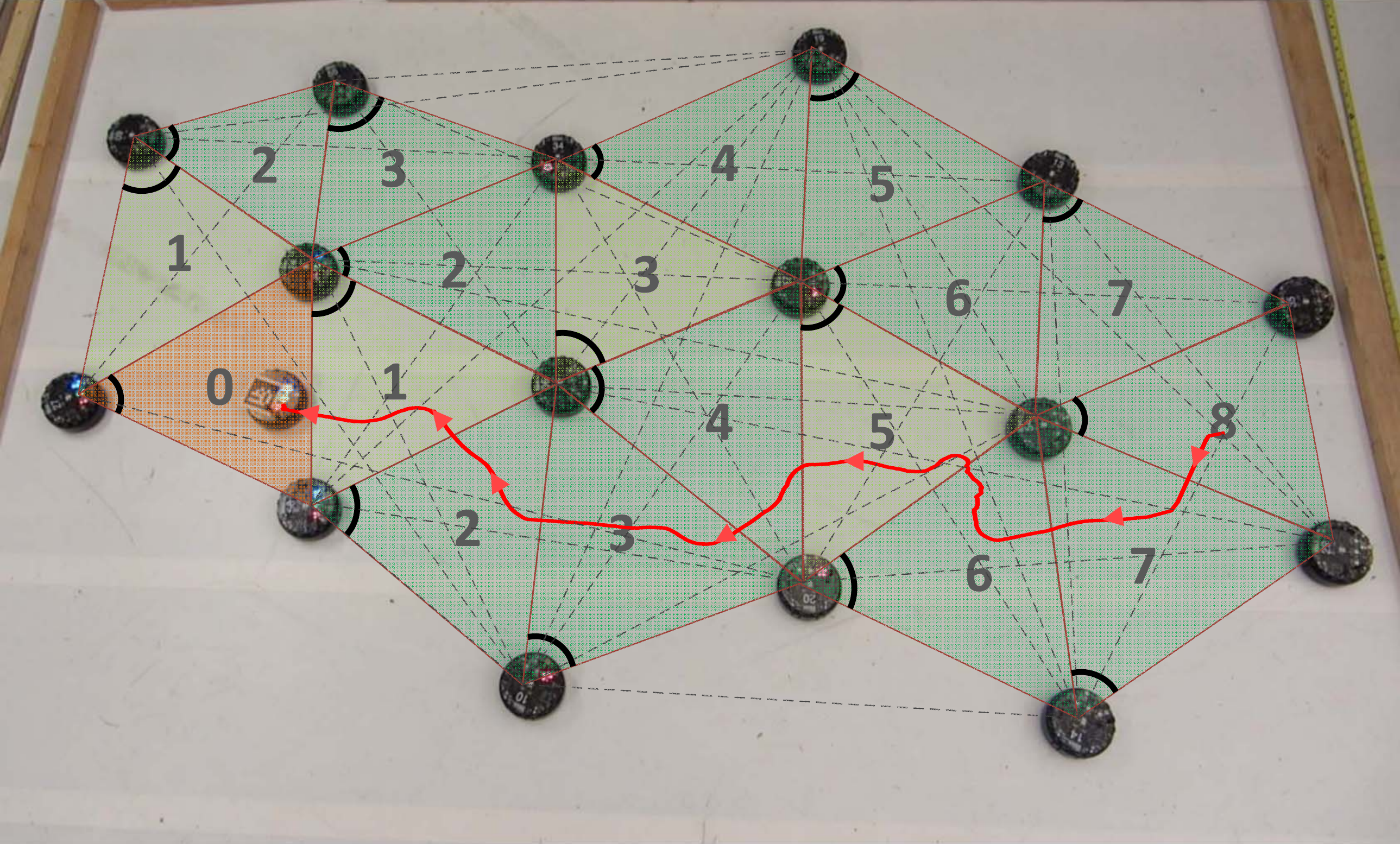}
\caption{
\label{fig:NavTestResult}
A sample triangulation and navigation experiment result with 17 r-one robots. The planar network of triangulation is a subset of the full network (dashed gray lines). Each robot creates a new triangle (dark green) by expanding toward the frontier or discovers new triangles (light green triangles) by examining local network geometry. Small black arcs indicate which robot creates and stores which triangle. This is a distributed ~\emph{physical data structure}; there is no centralized storage of triangulation information.
The network between adjacent triangles forms a \emph{dual graph} of the triangulation. In this example, a navigating robot uses a tree rooted at the red triangle to guide it from its start location to its current location. It follows the red path, the numbers indicate hops in the dual graph from the goal triangle.
}
\end{figure}


We are interested in solutions for large populations of robots, and focus our attention on approaches applicable on small, low-cost robots with limited sensors and capabilities.  In this work, we assume that robots do not have a map of the environment, nor the ability to localize itself relative to the environment geometry, \emph{i.e.} SLAM-style mapping is beyond the capabilities of our platform.  We exclude solutions that
use centralized control, as the communication and processing constraints do
not allow these approaches to scale to large populations.  We also do not assume that GPS localization or external communication infrastructure is available, which are limitations present in an unknown indoor environment.  Finally, we assume that the communication range is much smaller than the size of the environment, so a multi-hop network is required for communication, and the \emph{local network geometry} provides each robot with geometric information about its neighboring robots.

The basic problem requires exploring an unknown region by triangulation from a given starting position. The maximum edge length is a triangle is bounded by
the communications range of the robots. If the number of available robots is
not bounded a priori, the problem of minimizing their number for covering all
of the region is known as the {\em Minimum Relay Triangulation Problem}
(MRTP); if their number is fixed, the objective is to maximize the covered area, which is known as the {\em Maximum Area Triangulation Problem} (MATP). Both problems have been studied both for the {\em offline} scenario, in which the region is fully known, and the {\em online} scenario, where the region is not known in advance~\cite{fkk+-etrsr-11}.  Online MRTP admits a 3-competitive strategy, while the online MATP does not allow a bounded competitive factor: If the region consists of many narrow corridors, we may  run out of robots exploring them, and thereby miss a large room that could permit large triangles.
In this work, we focus on the online MATP problem. Our algorithm proceeds by extending the covered region by adding new triangles to the frontier of the exploration. The motion controllers we present use local geometric information. In particular, we focus on a simple platform that can only measure angles between neighbors and detect nearby obstacles. We provide a number of results:

\begin{itemize*}
\item We develop simple and efficient exploration methods based on triangulation.
\item We show that these methods only require local information and geometry.
\item We demonstrate that well-known abstract concepts (such as the dual graph) can be implemented in a distributed network of robots.
\item We provide provable performance guarantees for online triangulation and routing.
\item We demonstrate the practicality of our method by implementing it with simple, low-cost robots. (See our video~\cite{bfk+-tueur-13} for an overview of \cite{fkk+-etrsr-11}.)
\end{itemize*}



\subsection*{Related Work}
Our work combines ideas of self-organization and routing in stationary sensor networks with approaches to dynamic robot swarms. For the former, e.g., see \cite{fk-gbrls-06,fk-trsn-07}. The new challenges arise from considering a large number of mobile nodes with limited capabilities, and real-life platforms and constraints.  Classical triangulation problems seek a triangulation of all vertices of a polygon, but allow arbitrary length of the edges in the triangulation~\cite{fkk+-etrsr-11}. This differs from our problem, in which edge lengths are bounded by communication length.
Triangulations with shape constraints for the triangles and the use of Steiner points are considered in mesh generation, see for example the survey by Bern
and Eppstein~\cite{be-mgot-92i}.

The problem of placing a minimum number of relays with limited communication
range in order to achieve a connected network (a generalization of the
classical Steiner tree problem) has been considered by Efrat et
al.~\cite{efgmps-iaarp-08}, who gave a number of approximation results for the offline problem (a 3.11-approximation for the one-tier version and a PTAS for
the two-tier version of this problem); see the survey \cite{dfk-srpra-11} for related problems.  A similar question was considered by Bredin et al.~\cite{bdhr-dsngf-10}, who asked for the minimum number of relays to be placed to assure a $k$-connected network. They presented approximation results for the offline problem. 
For swarms, Hsiang et al.~\cite{habfm-ardrs-04}
consider the problem of dispersing a swarm of simple robots in a cellular
environment, minimizing the time until every cell is occupied by a robot. For
workspaces with a single entrance door, Hsiang et al.~present algorithms with
time optimal \emph{makespan} and $\Theta(\log(k+1))$-competitive algorithms for $k$ doors.
%

There are many types of exploration in the multi-robot literature. McLurkin and Smith~\cite{ms-dadie-04} present a breadth-first distribution from a more practical view, using a swarm of 100 robots.
Durham et al.~\cite{durhamBullo} present an algorithm for a team of robots to cover entire region using pursuit-evasion problem, using the robot's state to store intermediate results of the search.
Spears and Spears describe slgorithms for to produce a triangle lattice, but this is not a \emph{triangulation}: there is no knowledge of triangles, the dual graph, or distributed data structures for computation\cite{spears_overview_2005}.

\todo{add centibot ref, spears ref}


%% file: 02-ModelAndPreliminaries.tex
\section{Computational Model}
\label{sec:ModelAndAssumptions}

We have a system of $n$ robots.  The communication network is an undirected graph $G=(V,E)$. Each robot is modeled as a vertex, $u \in V$, where $V$ is the set of all robots and $E$ is the set of all robot-to-robot communication links. The neighbors of each vertex $u$ are the set of robots within line-of-sight communication range $r_{max}$ of robot $u$, denoted $N(u)=\{v \in V\ \mid \{u,v\} \in E\}$. Robot $u$ sits at the origin of its local coordinate system, with the $\hat{x}$-axis aligned with its current heading.  Each robot can measure the angles of the geometry of its local network, as shown in Fig.~\ref{fig:BearingOrientDefinition}. Robot $u$ cannot measure distance to its neighbors, but can only measure the \emph{bearing} and \emph{orientation}. We assume that these angular measurements have limited resolution.

\begin{figure}[t]
  \renewcommand{\figDim}{0.7}
  \centering
  \subfloat[][Bearing and orientation definitions.]{
  \label{fig:BearingOrientDefinition}
  \includegraphics[scale=\figDim]{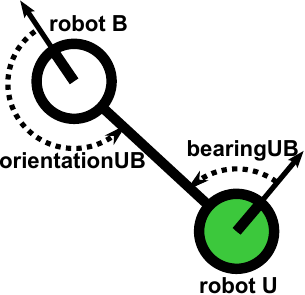}}
  \hspace{16pt}
  \subfloat[][Measuring triangle angles using 2-hop  coordinate system.]{
  \label{fig:TriangleAngleDefinition}
  \includegraphics[scale=\figDim]{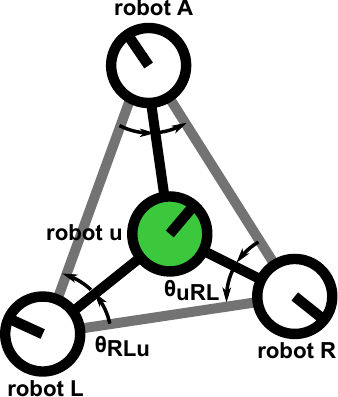}}
  \caption{\textbf{\protect\subref{fig:BearingOrientDefinition}} Robot $u$ can measure the bearing to neighbor $u_L$, $B_u(u_L)$, and the orientation of neighbor $u_L$, $Ori(u_L)$. \textbf{\protect\subref{fig:TriangleAngleDefinition}}, triangle angles (black arrows) are measured from neighbors of robot $u$, and shared with $u$ using a local broadcast message. We define left (right) inner angle for the triangle angle of $u$'s left (right).} 
\end{figure}
\setlength{\textfloatsep}{5pt}
Robots share their angle measurements with their neighbors. In this way, robot $u$ can learn of all angles in its 2-hop neighborhood. Fig.~\ref{fig:TriangleAngleDefinition} shows the relevant \emph{inner angles} of a triangle around $u$.  Each neighbor of $u$ computes these angles from local bearing measurements, then announced them.  The communication used by these messages is $O(max(\delta(u \in V)^2)$, where $\delta(u)$ is the degree of vertex $u$.

Each robot has contact sensors that detect collisions with the environment. There is an obstacle avoidance behavior that can effectively maneuver the robot away from these collisions. The robots also have a short-range obstacle sensor that can detect walls closer than  $\approx50$ cm.  The obstacle sensor does not detect neighboring robots.

Algorithm execution occurs in a series of synchronous \emph{rounds}, $t_{r}$.  This greatly simplifies analysis and is straightforward to implement in a physical system~\cite{mclurkin_analysis_2008}.  At the end of each round, every robot $u$ broadcasts a message to all of its neighbors. The robots randomly offset their initial transmission to minimize collisions.
During the duration of each round, robot $u$ receives a message from each neighbor $v \in N(u)$. Each message contains a set of public variables, including the sending robot's unique ID number $u.id$. The remaining variables will be defined later, but we note that the number of bits needed for each variable is bounded by $log_2 n$, i.e. the number of bits required to identify each robot. This produces a total message of constant size.

%% file: 03-MATAlgorithm.tex
\section{Max-Area Triangulation Algorithm}
\label{sec:MATAlgorithm}

%


Fig.~\ref{fig:bfs} illustrates the execution of the Max-Area Triangulation (MAT) algorithm. Initially, two~\emph{base robots} mark the~\emph{base edge} --- such as a door to an unexplored building. The algorithm starts with this base edge and proceeds by constructing a triangulation in a breadth-first manner.  The triangulation is extended as robots construct triangles along the current \emph{frontier} of exploration. The frontier is shown as blue lines in Fig.~\ref{fig:bfs}, and it delineates the boundary between triangulated space and untriangulated space.  All the area between the base edge and the frontier is triangulated. Each mobile robot extends the frontier by moving into unexplored space and forming a triangle with itself and at least two other adjacent robots from the frontier.  The algorithm terminates when either all of the workspace has been explored, or the maximum number of robots has been exhausted.

Each robot tries to build a \emph{high-quality} triangle---one that does not have edges that are too short or angles that are too small.  Equilateral triangles are ideal, but cannot always be constructed due to errors or environmental constraints.

\begin{figure}[t]
  \renewcommand{\figheight}{3.0in}
  \centering
  \includegraphics[width=.45\linewidth]{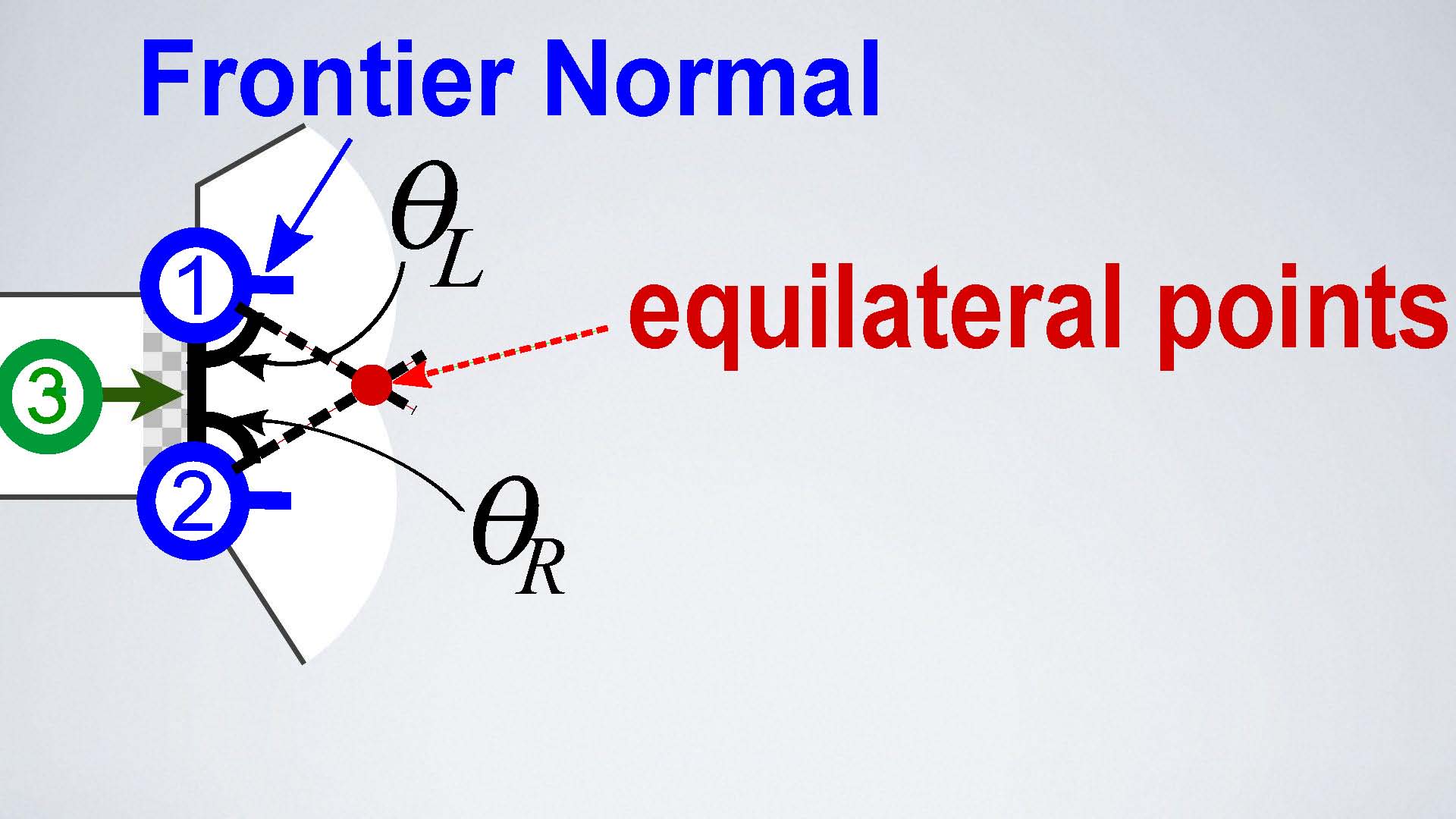}
\vspace*{1mm}
  \includegraphics[width=.45\linewidth]{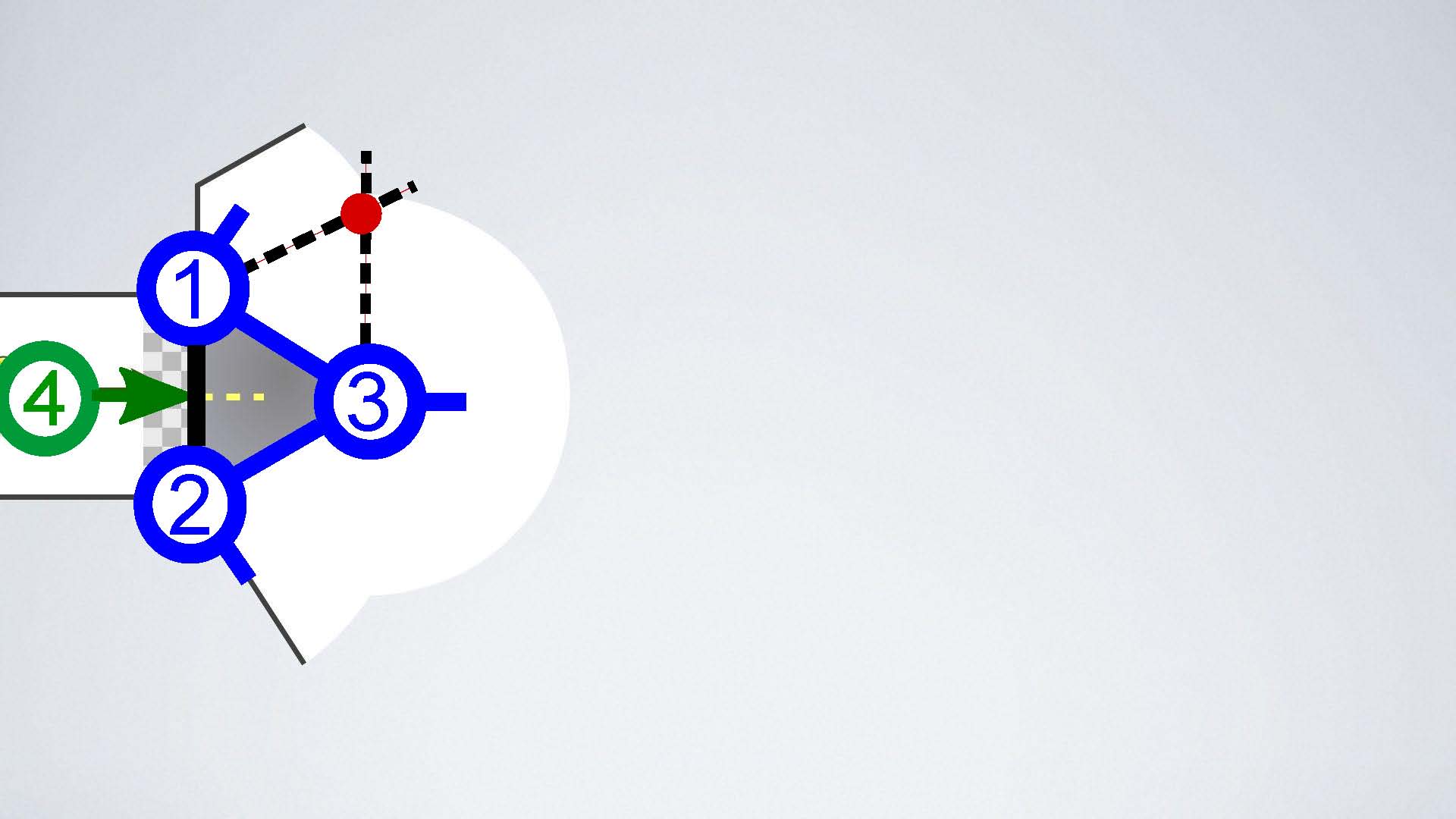}
\vspace*{1mm}
  \includegraphics[width=.45\linewidth]{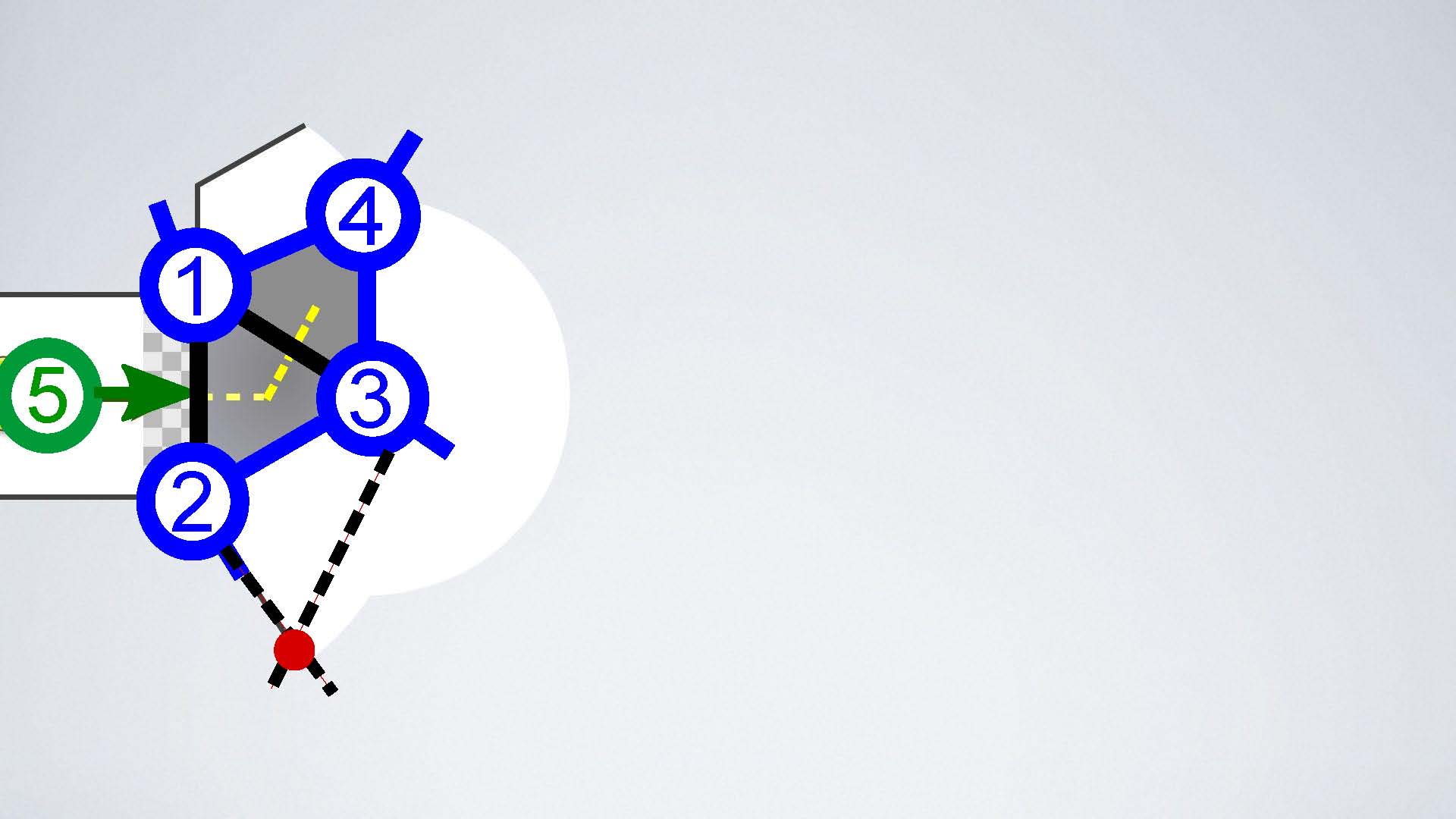}
\vspace*{1mm}
\includegraphics[width=.45\linewidth]{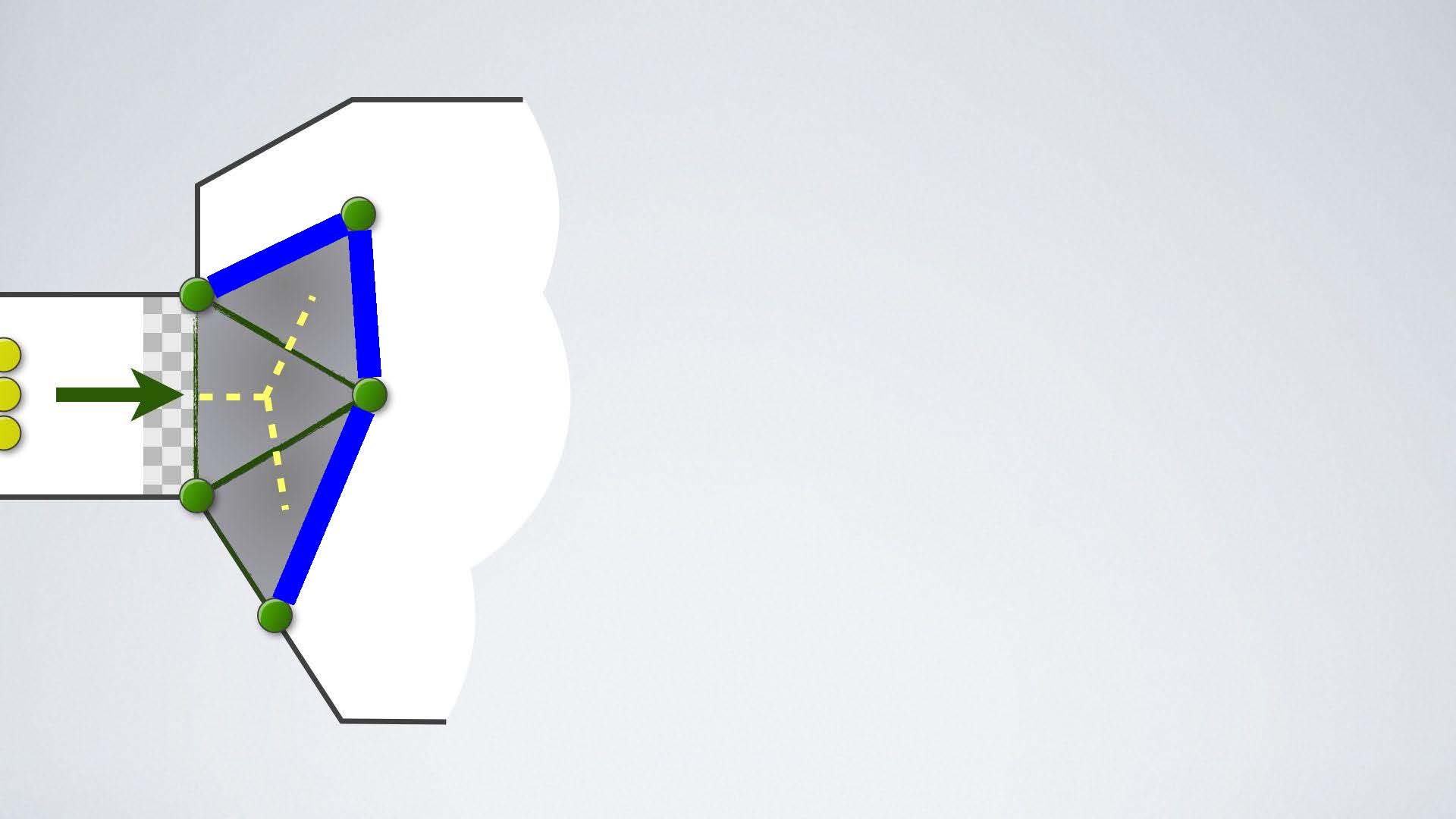}
\vspace*{1mm}
 \includegraphics[width=.45\linewidth]{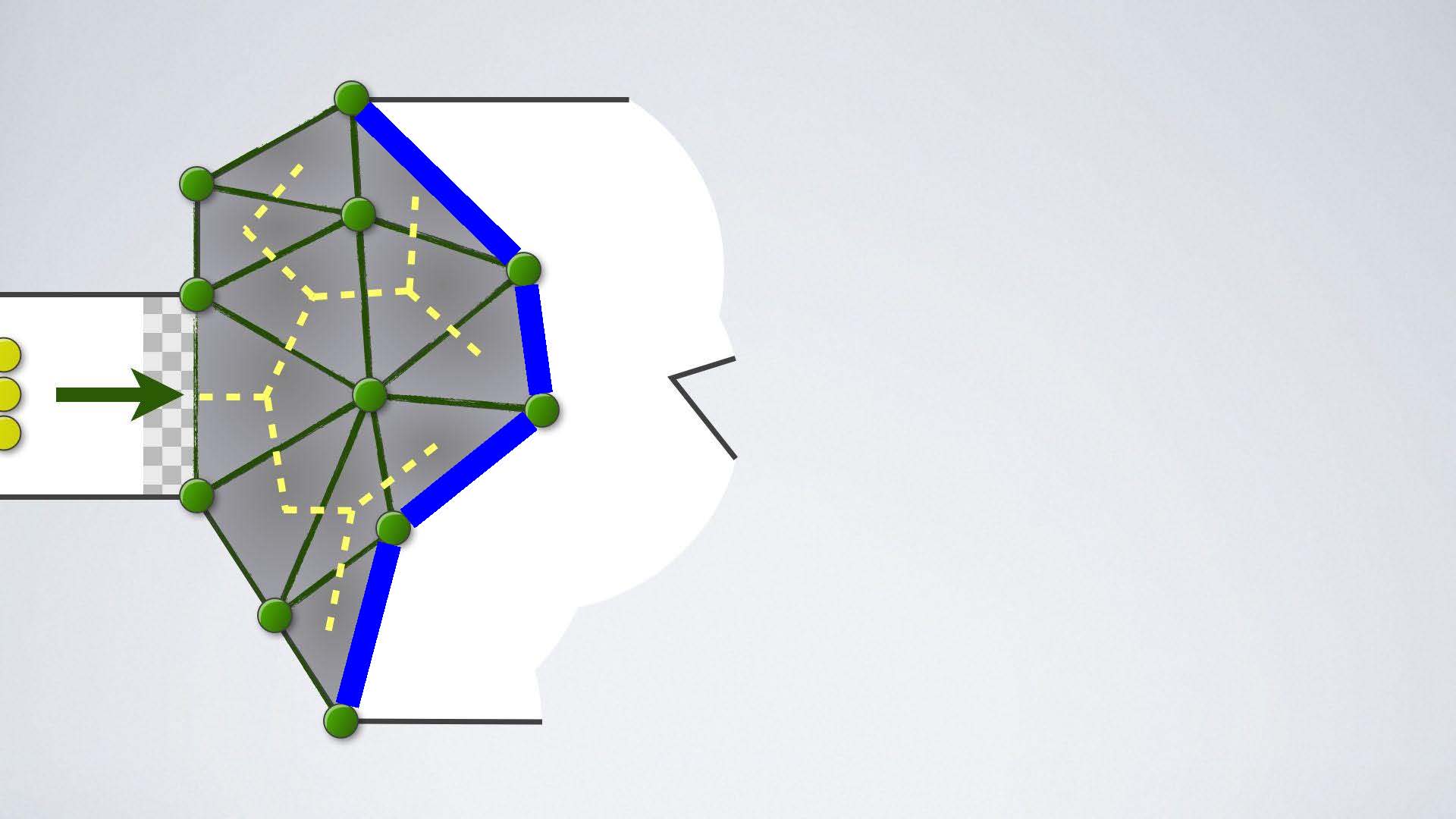}
\vspace*{1mm}
  \includegraphics[width=.45\linewidth]{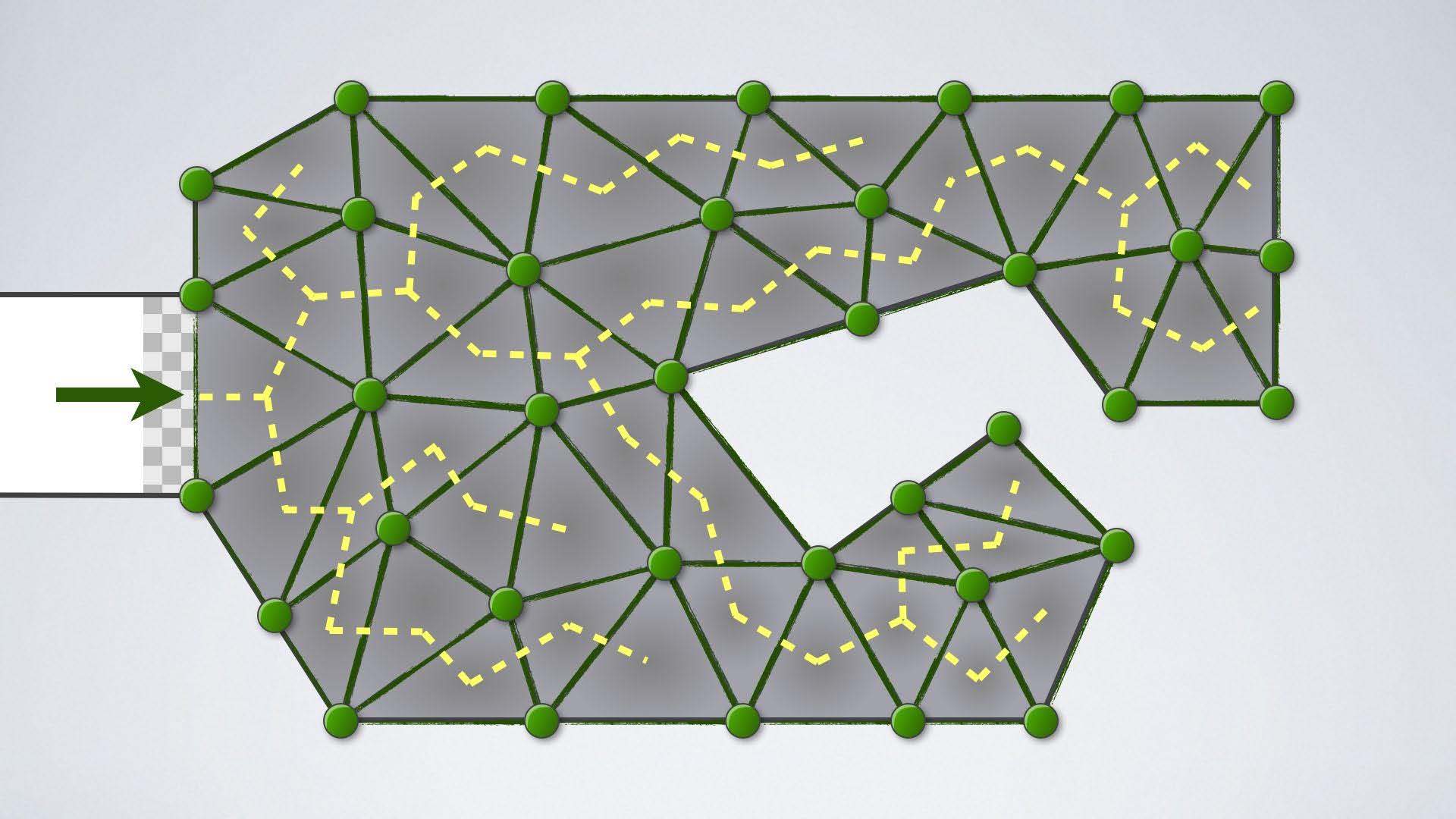}
  \caption{Constructing a triangulation in a BFS manner. The frontier edges are blue and interior edges are green. The edges of the BFS tree (the dual graph) are yellow.  The blue tick marks on each robot show the direction of the \emph{frontier normal}.  This points into unexplored space, in the direction perpendicular to the frontier edges incident at each robot.}
  \label{fig:bfs}
\end{figure}

%
%

During algorithm execution, we we distinguish the following types of edges in the robot network $G$: 1) \emph{Frontier edges} (Blue lines in Fig.~\ref{fig:bfs}), $\{u,v\}\in E_F$, which belong to only one triangle and have at least one vertex that is not in contact with the wall. 2) \emph{Internal edges}, $\{u,v\}\in E_I$\, which belong to two adjacent triangles. 3) \emph{Wall edges}, $\{u,v\}\in E_W$, which also belong to only one triangle, but both vertices of the edge are in contact with a wall. 
The yellow lines indicate the dual graph, $D$, which connects adjacent triangles.  We will address the detail of the dual graph in Section.~\ref{sec:DualGraph}.






\subsection{Triangulation}
\label{discoveryTri}
%

Construction of a new triangle begins with the addition of a new \emph{navigating robot}, $u$.  To build the triangulation in a breadth-first fashion, a~\emph{frontier triangle} is selected that is the minimum distance in the dual graph from the base triangle.  This triangle will have at least one frontier edge, we select it to be the \emph{goal frontier edge}, $\{l,r\}$.   The robot uses the dual graph to navigate to the frontier triangle, these algorithms are described in Secs.~\ref{sec:DualGraph} and~\ref{sec:DualGraphNavigation}.

A new triangle can be formed in two ways, \emph{expansion} or \emph{discovery}.  Fig.~\ref{fig:Expand} illustrates the construction of a triangle by expansion. When navigating robot $u$ is within the frontier triangle, it switches to the \emph{expanding state}, and moves towards the equilateral point for the new triangle. When $u$ crosses the frontier edge $\{l,r\}$, it creates a new expansion triangle $\Delta ulr$ ($l=l_0$ and $r = r_0$ in Fig.~\ref{fig:Expand}).  Once robot $u$ arrives at the equilateral point, it switches to the \emph{expanded state}, and adds $\Delta ulr$ to its list of triangles, becoming its owner. Edge $\{l,r\}$ becomes an internal edge, and robot $u$ broadcasts a message to neighbors $l$ and $r$, so that they update their right and left frontier neighbors to $u$. Because the edge $\{l,r\}$ is now internal, it is not used for expansion again, which prevents creating overlapping triangles.

When $u$ enters the expanded state, it needs to discover all of the unexpanded high-quality triangles adjacent to $\Delta ulr$.  Fig.~\ref{fig:Discovery} shows an example of triangle discovery.  We describe the process for the left frontier neighbor ($l$), it is analogous for the right.  We label the left neighbors $\{l_0, l_1, \ldots\}$ where $l_0 \equiv l$.  Robot $u$ first considers neighbor $l_1$, then proceeds through each neighbor on its left side in counter-clockwise order.  For each neighbor $l_i, i\geq 1$, robot $u$ checks for edge $\{l_i, l_{i-1}\} \in E_F$.  If this edge exists, then $u$ forms a candidate triangle, $\Delta ul_il_{i-1}$ (light green in Fig.~\ref{fig:Discovery}), and evaluates its quality using definition~\ref{def:alphaFat}.  The search terminates if the triangle is not high-quality, or there are no further neighbors to consider.  If the candidate triangle is high-quality, robot $u$ becomes its owner, and switches its left frontier neighbor from $l_{i-1}$ to $l_i$. Robot $u$ then broadcasts a message to $l_i$ to update its right frontier neighbor from $i_{i-1}$ to $u$.

\begin{figure}[t]
\renewcommand{\figDim}{1.1}
\centering
  \subfloat[][Forming an Expansion type of triangle.]{
  \label{fig:Expand}
  \includegraphics[width = .45\linewidth]{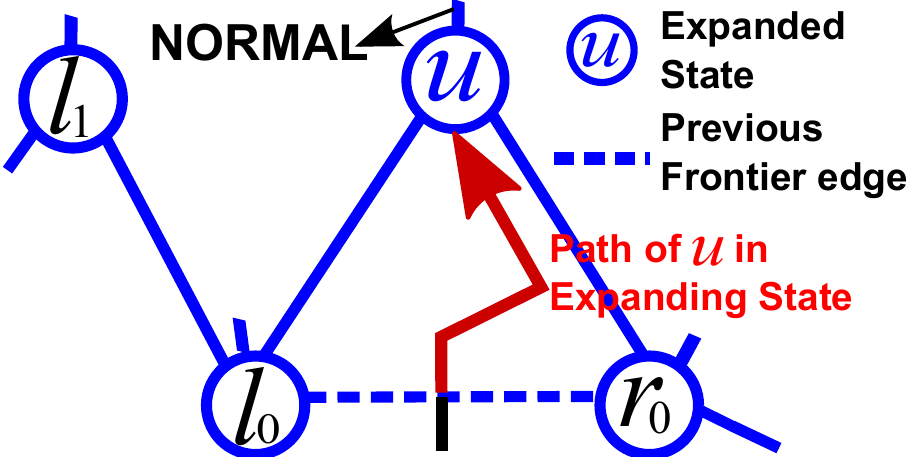}}
  \hspace{3pt}
  \subfloat[][Forming a Discovery type of triangle.]{
  \label{fig:Discovery}
  \includegraphics[width=.4\linewidth]{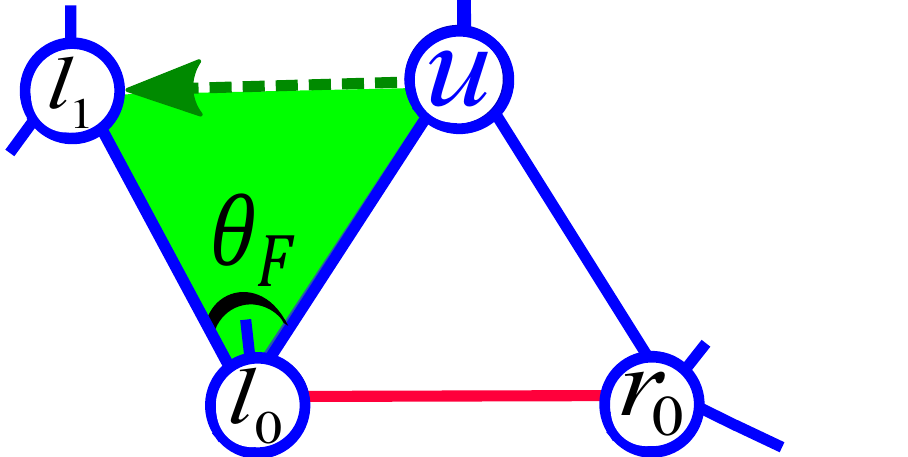}}
\caption{
\textbf{\protect\subref{fig:Expand}} An example of \emph{Expansion} triangle: Red arrow is the path of $u$ in \emph{Expanding} state. After arriving at the equilateral point, $u$ switches to \emph{Expanded} state.
\textbf{\protect\subref{fig:Discovery}} An example of \emph{Discovery} triangle (light green). $u$ checks $\theta_F$ to evaluate the quality of candidate triange, $\Delta ul_1l_0$. 
}
\end{figure}

%

%
%
%

%% file: 04-VirtualDualGraph.tex
\subsection{Dual Graph Construction}
\label{sec:DualGraph}

The dual graph of our a triangulation, $D$, describes the adjacencies between adjacent triangles. The dual graph can be used for realizing global objectives, such as routing. However, one difficulty for a distributed swarm of robots is the absence of a centralized authority that can explicitly keep track of a dual graph, as there are only ``primal'' vertices, i.e., robots.  Our solution is to establish and maintain the dual graph implicitly, by assigning each triangle $\Delta$ to a unique robot ``owner'', $o(\Delta)$, and then mapping edges between triangles in the dual graph to edges between robots in the primal graph.

We first observe that all owners are connected because all robots in our network, with the exception of the base robots, are owners by construction; every time a new navigating robot is added to the network, it becomes the owner of at least one constructed triangle. A robot can own multiple discovered triangles, and must maintain multiple vertices in the dual graph.

We must ensure that two triangle owners connected by an edge in the dual graph can communicate with each other through the primal graph. This is trivial for two triangles $\Delta_1$ and $\Delta_2$ owned by the same robot,
so we must show that for two different triangle owners $o(\Delta_1)\neq o(\Delta_2)$ with a dual graph edge, $\{o(\Delta_1),o(\Delta_2)\}_D \in D$,  $\{o(\Delta_1),o(\Delta_2)\}$ is an edge in the primal graph.



\begin{lemma}
\label{OwnerInFrontierEdge}
Consider edge $\{a,b\} \in E_F$ and $\Delta abo$, where $o$ is the triangle owner. Then $o = a$ or $o = b$.
\end{lemma}


\begin{proof}
By contradiction: assume $o\neq a$ and $o\neq b$. Then consider the expanding state for $\Delta oab$. Since $o$ is the owner in the expanded state, $o$ must have been the navigation robot in the expanding state. Therefore $\{a,b\}$ was the frontier edge in the expanding state and $\{a,b\}$ is now the internal edge in the expanded state, a contradiction.
 \end{proof}


\begin{theorem}
\label{ConnectedOwners}
The owners of two adjacent triangles must also be connected.
\end{theorem}

\begin{proof}
Let $\Delta abc$ and $\Delta abd$ be the two adjacent triangles, and $\{a,b\}$ be the edge they share.
These two triangles can be formed in the following two ways (in the expanding state):
1) Robot $a$ was the navigation robot. Then $a$ is the owner for both $\Delta abc$ and $\Delta abd$. $a$ is connected to itself.  2) Robot $d$ was the navigation robot. This makes $\Delta abc$ an existing triangle and $\{a,b\}$ a frontier edge in the expanding state. $d$ is also the owner robot for $\Delta abd$ in the expanded state. Either $a$ or $b$ is the owner of $\Delta abc$ by Lemma~\ref{OwnerInFrontierEdge}, so $d$, the owner of $\Delta abd$, must be connected to the owner of $\Delta abc$ through either edge $\{a,d\}$ or edge $\{b,d\}$.
By symmetry, $b$ is equivalent to $a$ and $c$ is equivalent to $d$.
\end{proof}

%

%% file: 05-DualGraphNavigation.tex
\subsection{Dual Graph Navigation}
\label{sec:DualGraphNavigation}

We use the dual graph as a navigation guide for robots in our
triangulation. If the destination triangle is known, such as a frontier triangle, then a broadcast message can be used to build a BFS tree suitable for navigation\cite{li_navigation_2005}.  Our
previous work shows there is no lower bound on the competitive factor of the
stretch of a path in the online MATP problem~\cite{fkk+-etrsr-11}, but this
requires narrow corridors of infinitesimal width.
In the following, we show that more realistic assumptions do allow constant-factor performance.

Let $r_{\text{max}}$ be the maximum length of a triangulation edge.
We also consider a lower bound of $r_{\text{min}}$ on the length of the shortest edge in the triangulation; in particular,
we assume that the local construction ensures that any non-boundary edge is long enough to let a robot pass between the two robots marking the vertices
of the edge, so $r_{\text{min}}\geq 2\delta$, where $\delta$ is the diamater of a robot. (The practical validity of these assumptions for
a real-world robot platform will be shown in the experimental Section~\ref{sec:ExperimentalResults}.)
Finally, angular measurements of neighbor positions let us guarantee a minimum angle of $\alpha$ in all triangles.  These constraints give rise to the following:

\begin{definition}
\label{def:alphaFat}
Let $\mathcal T$ be a triangulation of a planar region $\mathcal R$, with vertex set $V$. $\mathcal T$ is $(\rho,\alpha)${\em -fat}, if
it satisfies the following properties:
\begin{itemize}
\item The ratio $r_{\text{max}}/r_{\text{min}}$  of longest to shortest edge in $\mathcal T$ is bounded by some positive $\rho$.
\item All angles in $\mathcal T$ have size at least $\alpha$.
\end{itemize}
\end{definition}

\noindent This definition is used to prove properties of triangulations.

\subsubsection{Covered Area}
\label{subsec:Covered}

\begin{theorem}
Consider a $(\rho,\alpha)$-fat triangulation of a set $V$ with $n$ vertices,
with maximum edge length $r_{max}$ and minimum edge length $r_{min}$.
Then the total triangulated area is within $\sqrt{3}\rho^2/2sin(\alpha)$
of the optimum.
\end{theorem}
\begin{proof}

Each edge has length at least $r_{min}$, and any angle is bounded from below by $\alpha$. The claim follows by trigonometry.
\end{proof}

Note that in a practical setting, $\rho$ will be much smaller than the theoretically possible worst case;
see Fig.~\ref{fig:SF} for a real-world evaluation.

\subsubsection{Path Stretch}
\label{subsec:PathStretch}

Now we establish that the dual graph of our triangulations can be exploited for provably good routing. We make use of the following terminology.

\begin{definition}
Consider a triangulation $\mathcal T$ of a planar region $\mathcal R$, with vertex set $V$. Let $s,g$ be points in $\mathcal R$
 and let $p(s,g)$ be a polygonal path in $\mathcal R$ that connects $s$ to $g$; let $d_{p}(s,g)$ be its length.
Let $\Delta_s$ and $\Delta_g$ be the triangles containing $s$ and $g$, respectively, and let
${D}(s,g):=\Delta_s, \Delta_1,\ldots,\Delta_{\ell},\Delta_g$ be a shortest path in the dual graph
of $\mathcal T$. Then a $\mathcal T${\em -greedy} path between $s$ and $g$ is a path $s,q_1,\ldots,q_{\ell},g$, such that
$q_i\in\Delta_i$, and consecutive vertices of the path are connected by a straight line.
\end{definition}


\begin{figure}[t]
\renewcommand{\figDim}{1.1}
\centering
  \subfloat[][ ]{
  \label{fig:greedy}
  \includegraphics[width = .6\linewidth]{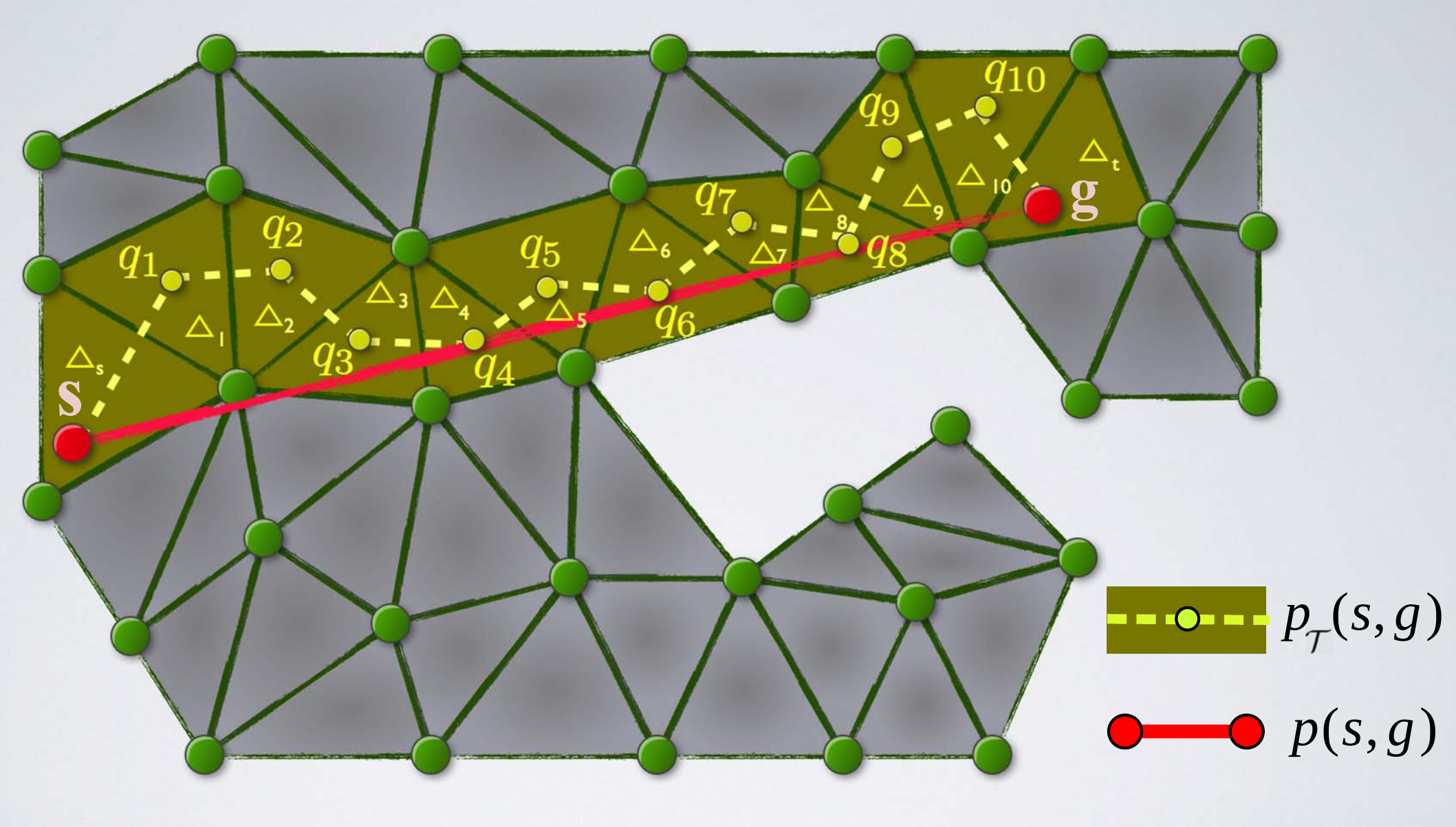}}
  \hspace{2pt}
  \subfloat[][ ]{
  \label{fig:cut}
  \includegraphics[width=.32\linewidth]{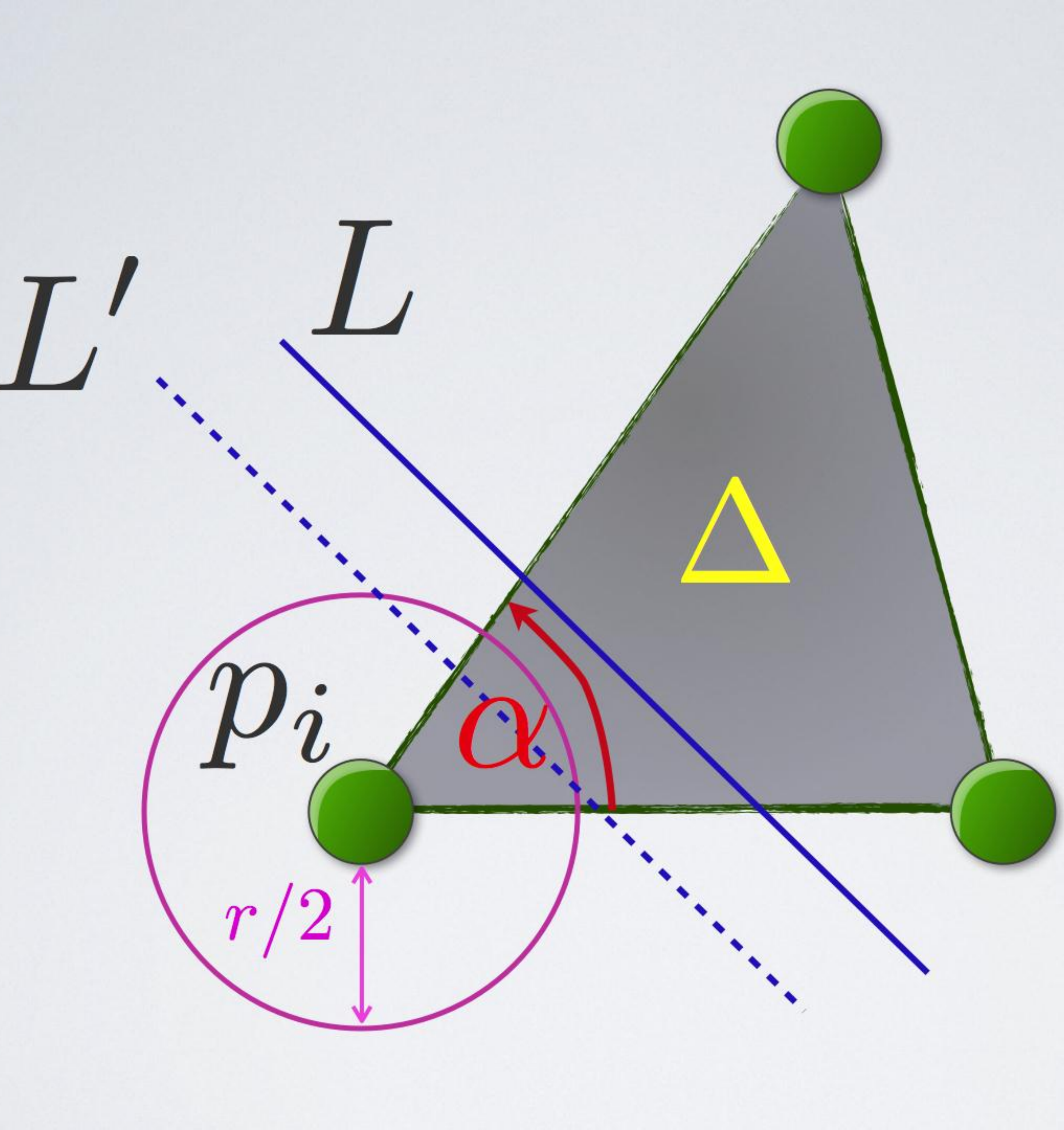}}
\caption{\textbf{\protect\subref{fig:greedy}} A shortest $s$,$g$-path (shown in red) in a region covered by a triangulation $\mathcal T$. The resulting $\mathcal T$-greedy path is depicted in yellow; a shortest dual path is indicated by colored triangles. Note that each point $q_i$ may be {\em anywhere} in the respective triangle $\Delta_i$.
\textbf{\protect\subref{fig:cut}} A triangle $\Delta$ is intersected by a straight line $L$. If $L$ passes the triangle not too close to one of the endpoints,
the length of the intersection is long. If the line passes the triangle close to one of the endpoints (indicated by the dashed line $L'$),
then the intersection with a circle of radius $r_{min}/2$ must be long.
}
\end{figure}

In other words, a $\mathcal T$-greedy path between $s$ and $g$ builds a short connection in the dual graph
of the triangulation, and then goes from triangle to triangle along straight segments. Note that we do not make any assumptions whatsoever concerning where we visit each of the triangles.

\begin{lemma}
Consider a $(\rho,\alpha)$-fat triangle $\Delta$ with minimum edge length at least $r_{min}$;
let $\Delta$ be intersected by a straight line $L$.
Then the total length of the intersection of $L$ and $\Delta$ is at least $\frac{r_{min}}{2\sin(\alpha/2)}$,
or the length of the intersection of $L$ with the $r_{min}/2$-disk around one of $\Delta$'s vertices
is at least $\frac{r_{min}}{2\sin(\alpha/2)}$.
\end{lemma}
%

\begin{proof}
Refer to Fig.~\ref{fig:cut}.
Consider the closest distance between $L$ and one of the vertices of $\Delta$. If this is larger than $2r_{min}{\cos(\alpha/2)}$,
then we see from Pythagoras' theorem that the intersection of $L$ and $\Delta$
must have length at least $2r_{min}{\sin(\alpha/2)}$.
Otherwise the distance is at most $2r_{min}{\cos(\alpha/2)}$, and
the intersection of $L$ with the $r_{min}/2$-disk around
the closest vertex of $\Delta$ must have length at least $2r_{min}{\sin(\alpha/2)}$.
\end{proof}

With this, we can proceed to the proof of the theorem.

\begin{theorem}
\label{thm:stretch}
Consider a $(\rho,\alpha)$-fat triangulation $\mathcal T$ of a planar region $\mathcal R$, with vertex set $V$, maximum and minimum
edge length $r_{max}$ and $r_{min}$, respectively.
Let $s,g$ be points in $\mathcal R$ that are separated by at least one triangle, i.e., the triangles $\Delta_s$, $\Delta_g$
in $\mathcal T$ that contain $s$ and $g$ do not share a vertex.
Let $p(s,g)$ be a shortest polygonal path in $\mathcal R$ that connects $s$ with $g$, and
let $d_{p}(s,g)$ be its length.
Let ${p_{\mathcal T}}(s,g)$ be a $\mathcal T$-greedy path
between $s$ and $g$, of length $d_{p_{\mathcal T}}(s,g)$. Then
$d_{p_{\mathcal T}}(s,g)\leq c \cdot d_{p}(s,g)+2$,
for $c=\lfloor\frac{2\pi}{\alpha}\rfloor\frac{\rho}{\sin(\alpha/2)}$,
and $d_{p_{\mathcal T}}(s,g)\leq c' \cdot d_{p}(s,g)$,
for $c'=\lfloor\frac{6\pi}{\alpha}\rfloor\frac{\rho}{\sin(\alpha/2)}$.
\label{theoremSF}
\end{theorem}

\begin{proof}
Consider ${p}(s,g)$, triangles $\Delta_s$, $\Delta_g$ and the sequence $\Delta_1,\ldots,\Delta_{\ell'}$
of $\ell'$ other triangles intersected by it; by assumption,
$\ell'\geq\ell\geq 1$, where $\ell$ is the number of triangles contained in ${p_{\mathcal T}}(s,g)$.
Furthermore, note that the disjointness of $\Delta_s$, $\Delta_g$ implies
$d_{p}(s,g)\geq r_{min}$.

We first show that $d_{p}(s,g)\geq \ell'\lfloor\frac{2\pi}{\alpha}\rfloor\frac{\rho}{2\sin(\alpha/2)}$.
For this purpose, charge the intersection of ${p}(s,g)$ with $\Delta_i$ to $\Delta_i$, if its length is
at least $\frac{\rho}{2\sin(\alpha/2)}$; if it is shorter, we charge the length of the intersection of
${p}(s,g)$ with the $r_{min}/2$-disk around one of $\Delta$'s vertices $p_j$ evenly to all of the triangles $\Delta_i$
that are incident to $p_j$. Because the minimum angle in a triangle is bounded from below by $\alpha$, the preceding lemma implies
the lower bound on the length of $d_{p}(s,g)$.

On the other hand, it is straightforward to see that no edge in a $T$-greedy $s,g$-path can be longer than $2r_{max}$.
Therefore, $d_{p_{\mathcal T}}(s,g)\leq 2(\ell+2)r_{max}$. Comparing the lower bound on $d_{p}(s,g)$ and the upper
bound on $d_{p_{\mathcal T}}(s,g)$ yields the claim $d_{p_{\mathcal T}}(s,g)\leq c \cdot d_{p}(s,g)+2$
with $c$ as stated. The additive term of 2 results from the $s$ and $g$ possibly being close to the boundaries of
$\Delta_s$ and $\Delta_g$, respectively; it can be removed by noting that $\ell'\geq\ell\geq 1$ implies
$(\ell+2)\leq 3\ell'$, as indicated by the second comparison and the choice of $c'$.
\end{proof}

This provides constant stretch factors even under minimal, purely theoretical and highly pessimistic assumptions.  The practical performance in real-world settings (where the greedy paths
do not visit worst-case points in the visited triangles) is considerably better, as we demonstrate in Section~\ref{sec:ExperimentalResults}.

%% file: 05-Implementaion.tex
\section{Implementation}
\label{sec:Implementation}
%
%
%

%
A high-level finite-state machine of our implementation of  triangulation construction is shown in Fig.~\ref{fig:MATFSM}.
Two robots are initialized in the Frontier-Wall state and placed at the base-edge. All other robots begin behind the base edge in the Navigation state. Table~\ref{table_hf} lists helper functions for all algorithms below.

\begin{figure}[t]
  \renewcommand{\figheight}{1.6in}
  \centering
  \includegraphics[width=\linewidth]{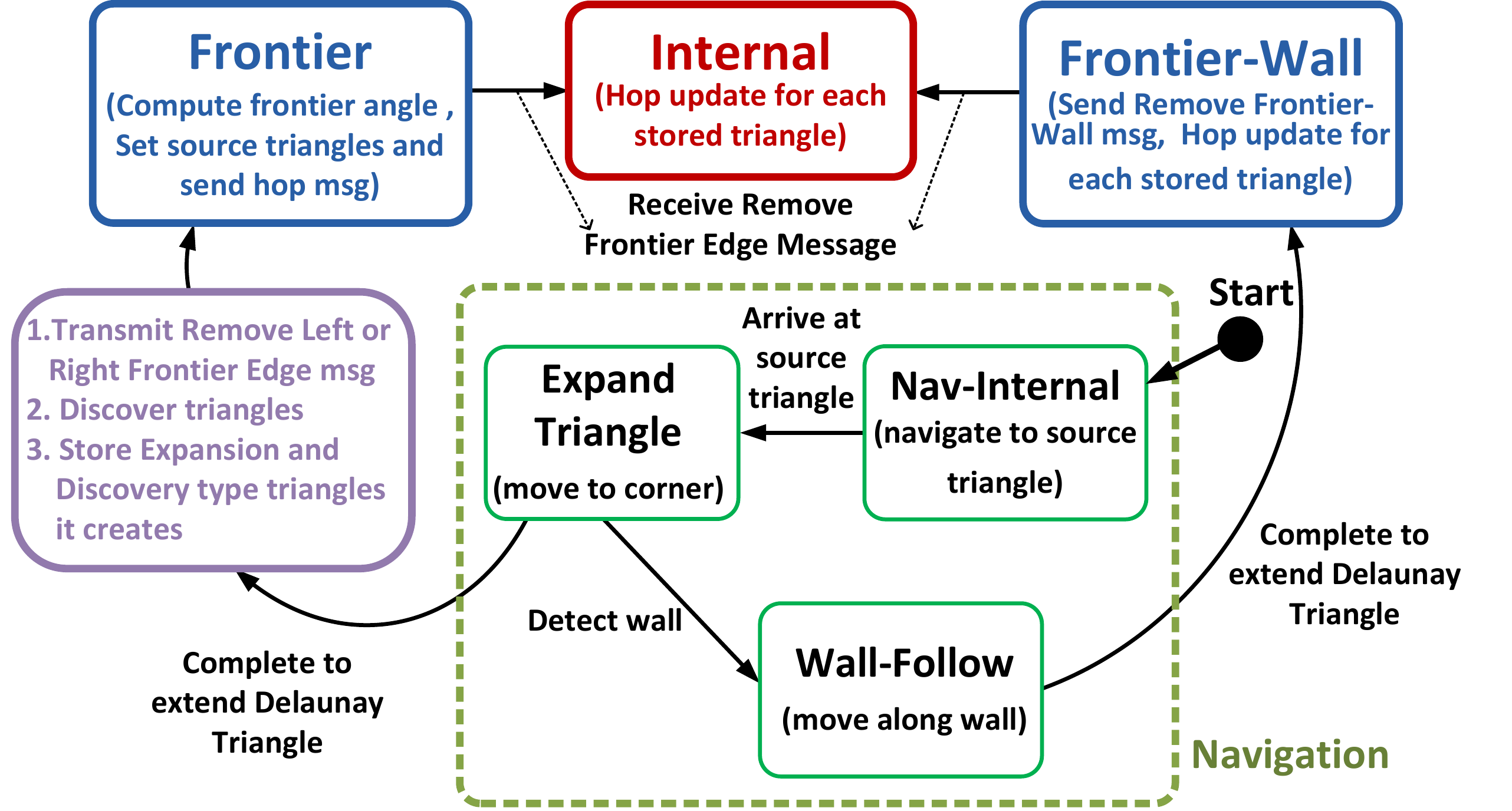}
  \caption{
    \label{fig:MATFSM}
    Finite state machine for MAT algorithm. }
\end{figure}

%



\subsection{Navigation State}

The navigation contains three states; \emph{Nav-Internal}, \emph{Expand-Triangle}, and~\emph{Wall-Follow}. A new robot, $u$, enters the network in the Nav-Internal state, and runs algorithm~\ref{alg:NavigatingState-NavInternal} to navigate to a frontier triangle.  Line 2 runs an occupancy test function, shown in Fig.~\ref{fig:occupancyTest}, that returns the current triangle, $T_c$, that contains robot $u$, and its owner, $o$. If $T_c$ is a non-frontier triangle, then $u$ moves to an adjacent triangle that is closer to (fewer hops from) the frontier (line 10 to 11).
Theorem~\ref{ConnectedOwners} ensures that the owner of $T_c$ is connected to owners of adjacent triangles, so $u$ learns the hops of all adjacent triangles with a 2-hop message similar to the geometry message from Fig.~\ref{fig:TriangleAngleDefinition}.
If $T_c$ is a frontier triangle (line 3) or null (only true if $u$ has just crossed the base edge, line 6), then $u$ will create a new triangle. The variables $u.L$ and $u.R$ are set to the left and right neighbors of the frontier edge (line 4 and 7), and the robot changes its state to Expand-Triangle (line 5 and 8).

\begin{algorithm}
\footnotesize
\caption{\proc{Nav-Internal}
\label{alg:NavigatingState-NavInternal}
}
\begin{algorithmic}[1]
\While{$u.state$ =\emph{Navigate-Internal}}
	\State $\mathcal T_c \gets \proc{GetCurrentTriangle()}$
	\If { $\proc{IsFrontierTriangle}(\mathcal T_c)$}
		\State $(u.L,u.R) \gets  \proc{GetFrontierEdgeNbr}(\mathcal T_c) $
		\State $u.state$ $\gets$ \emph{Expand-Triangle}	
	\ElsIf{$\proc{IsOnlyBaseEdge}(N(u))$}
		\State $(u.L,u.R) \gets  \proc{GetBaseEdgeNbr}() $
		\State $u.state$ $\gets$ \emph{Expand-Triangle}
	\Else
		\State $\mathcal T_{next} \gets$  $\proc{GetMinHopAdjTri}(\mathcal T_c)$
		\State $\proc{MoveToNextTriangle}(\mathcal T_{next})$
	\EndIf
\EndWhile
\end{algorithmic}
\end{algorithm}
\setlength{\intextsep}{0pt}
%
%

Once in the Expand-Triangle state, $u$ runs algorithm~\ref{alg:NavigatingState-Expand}. Line 2 computes the left and right inner angles to the frontier neighbors, $\theta_L$ and $\theta_R$. Line 3 then runs the triangle-expansion controller illustrated in Fig.~\ref{fig:vectorField} until $u$ is in region 3.

We lack the space here for a complete description of the controller, we sketch its operation here.  When robot $u$ enters region 3, if $\theta_L > \theta_R$, $u$ first moves toward $B_u(u.L)+\pi$ until $\theta_R \geq \frac{\pi}{3}$.  It then changes its heading toward $B_u(u.R)+\pi$, and moves until it reaches the goal region (region 4).  The opposite control happens when $\theta_L < \theta_R$.

Robot $u$ stores the triangle on its list (line 5), runs the Discover Triangle procedure to discover all adjacent triangles as described in section~\ref{discoveryTri} (line 6). The \emph{Frontier angle}, $\theta_F$, provides a simple way to evaluate the quality of candidate triangles; we define a triangle to be high-quality if $\theta_F < k$, with $k$ manually tuned to reduce errors.  After adding triangles, $u$ updates its frontier neighbors, those of $u.L$ and $u.R$ (line 7), adds new frontier neighbors in $u^{new}$, and disconnects frontier neighbors in $u^{old}$.  


%
%
If $u$ detects a wall while expanding a triangle (line 12), it changes its state to Wall-Follow (line 13). This controller moves $u$ along the wall until it forms an isosceles triangle.  Then $u$ stores the triangle, and broadcasts disconnect message to $u.L$ or $u.R$, and changes its state to Frontier-Wall.
%
\begin{table}[!t]
\scriptsize
\renewcommand{\arraystretch}{1}
\caption{Table of Helper Functions}
\label{table_hf}
\centering
\hspace{-.2cm}
\begin{tabular}{|l||l|}
\hline
$GetCurrentTriangle()$ & \parbox{\tablelength}{Runs occupancy test and returns current triangle, $T_c$.}\\
\hline
$GetMinHopAdjTri(\mathcal T_c)$ & \parbox{\tablelength}{Get $T_c$'s min-hop adjacent triangle.}\\
\hline
$DiscoverTriangle(u.L, u.R)$ & \parbox{\tablelength}{ Runs discovery procedure and gets discovery triangles, $T_D$, and list of $u$'s old and new frontier neighbors.} \\
\hline
$IsIsoscelesTriangle(u.L, u.R)$ & \parbox{\tablelength}{Checks if $\theta_L = \theta_R$  in an expand triangle.}\\
\hline
$GetFrontierWallNbr(u.L, u.R)$ & \parbox{\tablelength}{Returns $u.L$ or $u.R$ in frontier-wall state.}\\
\hline
$BCastFMsg(u^{old}, u^{new})$ & \parbox{\tablelength}{Broadcast new frontier msg to nbrs $\in u^{new}$.}\\
\hline
$RecvFMsg(u^{old}, u^{new})$ & \parbox{\tablelength}{Receive new frontier nbrs.}\\
\hline
$UpdateFNbr(u^{old}, u^{new})$ & \parbox{\tablelength}{ Change frontier nbr from  $u^{old}$ to $u^{new}$.}\\
\hline
$BCastDisconnectMsg(u^{old})$ & \parbox{\tablelength}{Broadcast disconnect msg to nbrs $ \in u^{old}$}. \\
\hline
$RecvDisconnectMsg()$ & \parbox{\tablelength}{ Return $u_{sender}$ if $u_{sender}$ disconnects $u$.}\\
\hline
$IsContainFrontierEdge(T_i)$ & \parbox{\tablelength}{Checks if $T_i$ has a frontier edge.}\\
\hline
$UpdateTriangleHop()$ & \parbox{\tablelength}{For each triangle $u$ owns, sets its hop to 1 + minimum among all adjacent triangles' hops.} \\
\hline
$BCastTriangleHop(N(o))$ &  \parbox{\tablelength}{Broadcast all hops of all triangles $u$ owns.}\\
\hline
\end{tabular}
\end{table}


\begin{algorithm}
\footnotesize
\caption{\proc{Expand-Triangle}
\label{alg:NavigatingState-Expand}
}
\begin{algorithmic}[1]
\While{$u.state$ = \emph{Expand-Triangle}}
	\State $(\theta_{L}, \theta_{R}) \gets \proc{GetInnerAngle}(u.L,u.R)$
	\State $\proc{TriangleExpansionController}(\theta_L, \theta_R)$
	\If {$\proc{IsInGoalRegion}$$(\theta_L, \theta_R)$}
		\State $\proc{StoreTrianglesToList}$ ($ \Delta uu.Lu.R$)
		\\ \ \ \  $(u^{old}, u^{new}, \mathcal T_D) \gets  \proc{DiscoverTriangle}(u.L, u.R) $
		\State $\proc{UpdateFNbr}(u^{old}, u^{new}) $
		\State $\proc{BCastFMsg}$($u^{old}$, $u^{new}$)
		\State $\proc{BCastDisconnectMsg}$($u^{old}$)			
		\State $\proc{StoreTrianglesToList}$ ($\mathcal T_D$)
		\State $u.state \gets$ \emph{Frontier}
	\ElsIf {\proc{IsWallDetected()}}
		\State $u.state$ $\gets$ \emph{Wall-Follow}	
	\EndIf
\EndWhile
\end{algorithmic}
\end{algorithm}

\begin{figure}[t]
  \renewcommand{\figDim}{1.1}
  \centering
  \subfloat[][Triangle Occupancy]{
  \label{fig:occupancyTest}
  \includegraphics[width=.35\linewidth]{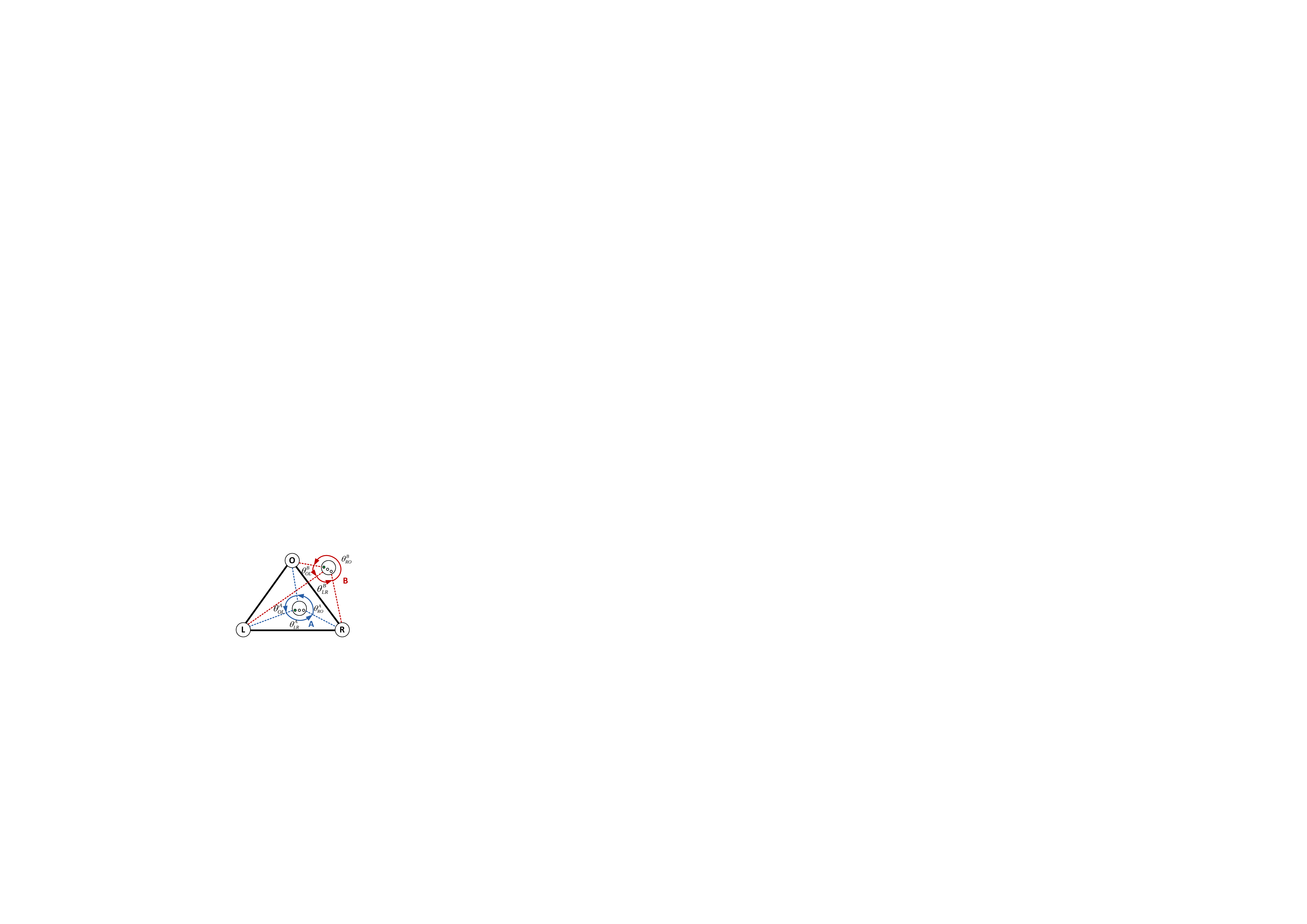}}
  %
  \subfloat[][Triangle Expansion Controller]{
  \label{fig:vectorField}
  \includegraphics[width=.6\linewidth]{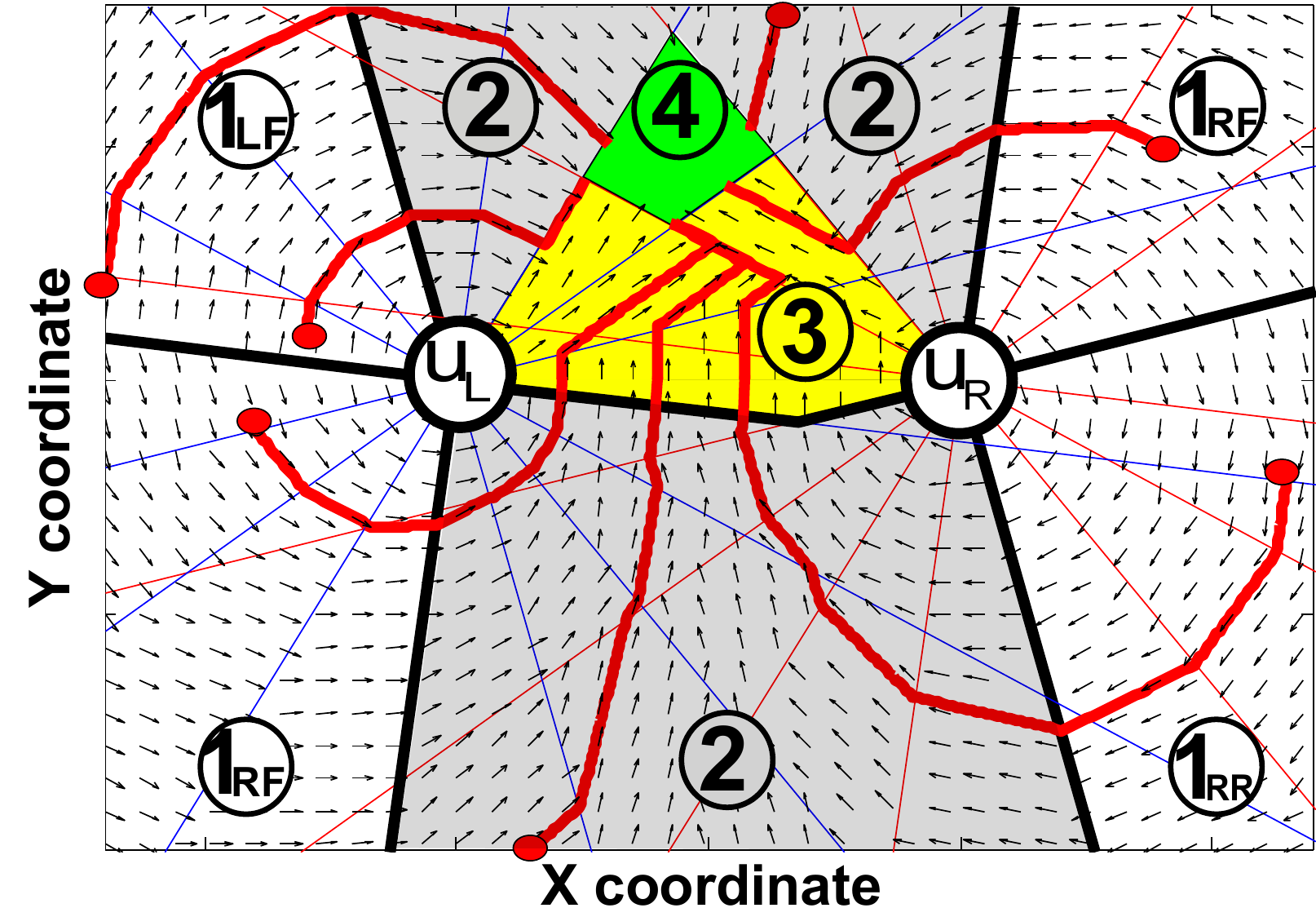}}
  \caption{
  \textbf{\protect\subref{fig:occupancyTest}} The occupancy test algorithm determines if a robot is inside a given triangle.  If any angle between neighbors of $u$ is greater than $\pi$, then $u$ is outside of the triangle.
  \textbf{\protect\subref{fig:vectorField}}
  Diagram of triangle expansion controller regions between robots $u_L$ and $u_R$, each with $\frac{\pi}{8}$ bearing resolution.  A robot 1 in region 1 rotates around $u_L$ or $u_R$, so that it always converges in region 2. A robot in region 2 always moving toward the direction where its two inner angles are getting smaller. By doing so, The controller always guides a robot from lower number region to higher number adjacent region by only considering $\theta_L$ and $\theta_R$.
All sample trajectories in red lines converge to the goal region.
}
\end{figure}

\subsection{Frontier and Frontier-Wall State}
When robot $u$ enters the frontier or frontier-wall state, it becomes stationary and runs algorithm~\ref{alg:FrontierState}.
%
In lines 3-7, $u$ labels all of its triangles which include a frontier edge to frontier triangles.  These triangles become sources for the frontier message that guides navigating robots to them. (line 8-9). 
%
The frontier robots compute the \emph{Frontier angle}, $\theta_F$, between adjacent frontier neighbors, in the direction of the frontier normal, shown in Figs.~\ref{fig:Discovery} and~\ref{fig:bfs}.
This is done in line 11-12.
%
%
%
%
To maintain the simply-connected frontier subnetwork, robot $u$ will need to update its frontier edges when new triangles are added.  After a new navigation robot, $v$, expands and discovers new triangles, lines 14-16 ensure the frontier edges adjacent to robot $u$ are updated when messages from $v$ are received.  If robot $u$ receives a disconnect message from $v$, it has no more incident frontier edges, and transitions to the Internal state in lines 17-19.  This is illustrated in Fig.~\ref{fig:Discovery}) by robot $l_0$.
If $u$ and $v$ are both frontier-wall robots, $v$'s disconnect message to $u$ will cause it to transition to internal state and create a wall edge (line 20-22). 

\begin{algorithm}
\footnotesize
\caption{\proc{Frontier$/$Frontier-Wall}()
\label{alg:FrontierState}}
\begin{algorithmic}[1]
\While{$u.state $=\emph{Frontier} OR $u.state $=\emph{Frontier-Wall}  }
\ForAll {$\mathcal T_i \in TriangleList$}
	\If {$\proc{IsContainFrontierEdge}$($T_i$)}
		\State $\proc{SetFrontierTriangle}$($T_i$)
	\Else
		\State $\proc{ClearFrontierTriangle}$($T_i$)
	\EndIf
	\State $\proc{UpdateTriangleHop}(T_i, N(o))$
	\State $\proc{BCastTriangleHop}(N(o))$
\EndFor
\State \proc{Compute/RotateToNormalVec}($B(u.L), B(u.R)$)
\State $\theta_F \gets$ \proc{ComputeFrontierAngle}($B(u.L), B(u.R)$)
\State \proc{BroadcastFrontierAngle}($\theta_F$)
\If{\proc{RecvFMsg}($u^{old}$, $u^{new}$)}
	\State UpdateFNbr($u^{old}$, $u^{new}$)
\EndIf
\State $v  \gets$ \proc{ RecvDisconnectMsg}()
\If {$v.state$ = \emph{Frontier}}
	\State $u.state$ $\gets$ \emph{Internal}
\ElsIf {$v.state$ = \emph{Frontier-Wall} $\wedge u.state$= \emph{Frontier-Wall}}
	\State $u.state \gets$ \emph{Internal}
\EndIf
\EndWhile
\end{algorithmic}
\end{algorithm}

%

\subsection{Internal State}
Eventually, robot $u$ is likely to become an Internal robot.  It remains stationary, and relays broadcast messages.  Every robot processes broadcast messages by updating the hops of each triangle they own by considering the hops to adjacent triangles, finding the minimum, and adding one.  This procedure propagates the broadcast message, and the hops updated, and ensures that any new robot crossing the base edge will move to the frontier triangle that is nearest in the dual graph, providing a breadth-first construction.

%% file: 05-Experiments.tex
\section{Experimental Results}
\label{sec:ExperimentalResults}


We have performed  several real world experiments, using the r-one
robots shown in~\cite{dars-rone}.  The capabilities of this platform  supports the assumptions in our problem statement.; each robot can measure the bearings to its nearby robots, despite of a limited resolution of only $\frac{\pi}{8}$, and exchange messages including those bearings and necessary information to run an implemented algorithm in Section~\ref{sec:Implementation} using inter-robot communication. Each robot also has 8 bump sensors that provides wall detection.
To evaluate a resulting triangulation quality or trace the trajectory of a navigation robot, we use the April-Tag system by APRIL group~\cite{april}. This measures the ground truth position, $P_u = \{ x_u,y_u, \theta_u \}$, of each robot $u$.
The $u$, however, cannot measure or use the ground-truth position while executing our algorithms. All robots only know the two-hop local network geometry shown in Fig.~\ref{fig:TriangleAngleDefinition}.


\subsection{Maximum Area Trangulation using MATalgorithm}
Fig.~\ref{fig:Screenshot} shows snapshots of triangulation. Over 8 trials using 9-16 robots, the average triangulated area is 1.5$\pm$0.29$m^2$.  It takes 7.8$\pm$2.1 robots to cover a unit area ($1$m$^2$).  The resulting triangulations are  $(\rho=3.6, \alpha=0.36 rad)${\em -fat}.
Fig.~\ref{fig:pieChart} shows that our triangulations cover about 91\% of the region behind the frontier edges.  The uncovered region is because the top-left and bottom-left corner in Fig.~\ref{fig:Screenshot} are wall edges (incident on two wall robots), and are not expanded by navigating robots.
Fig.~ \ref{fig:Histogram} shows the distribution of area covered by individual triangles.  The initial length of the base edge predicts the area of an ideal equilateral triangle should be $0.088m^2$, our triangles have a mean area of $0.13m^2$, with a std. dev. of $0.065m^2$. This discrepancy caused by the angle-based sensors; the robots cannot measure range, and therefore cannot control the area of the triangle they produce.  We show this by studying individual triangle quality.

\begin{figure*}[t]
  \renewcommand{\figheight}{1.6in}
  \centering
  \includegraphics[width=.95\linewidth]{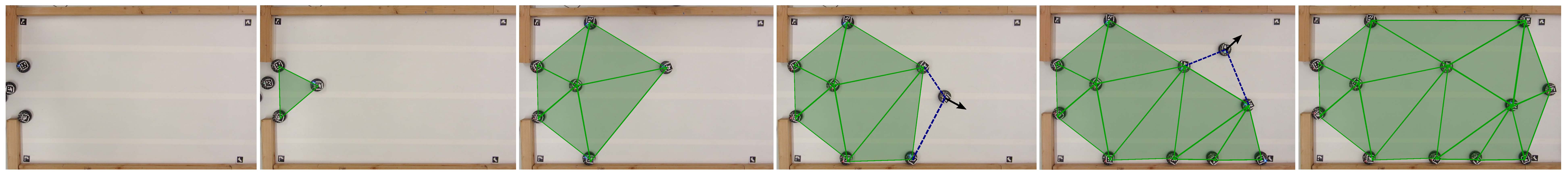}
  \caption{
   \label{fig:Screenshot}
  Screenshots of constructing triangulation with 12 robots.}
\end{figure*}

\begin{figure}[t]
\renewcommand{\figDim}{1.1}
\centering
  \subfloat[][]{
  \label{fig:pieChart}
  \includegraphics[scale=0.44]{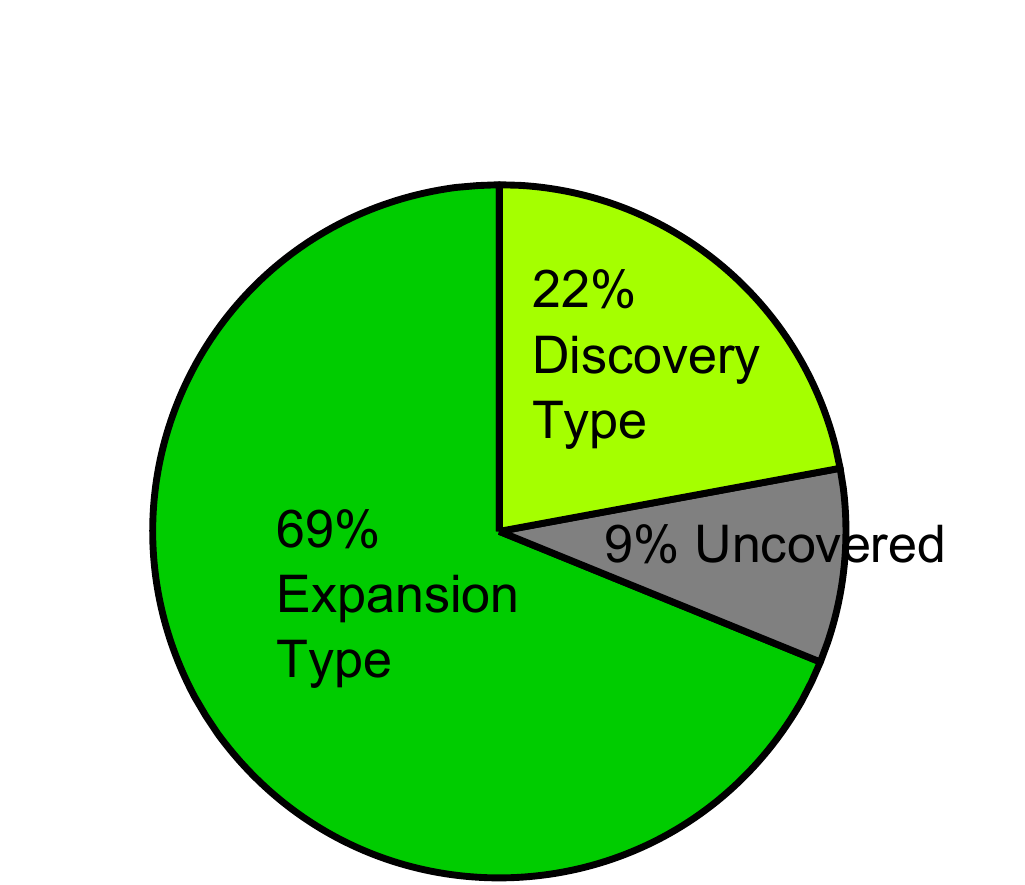}}
\ \ \
  \subfloat[][]{
  \label{fig:Histogram}
  \includegraphics[scale=0.36]{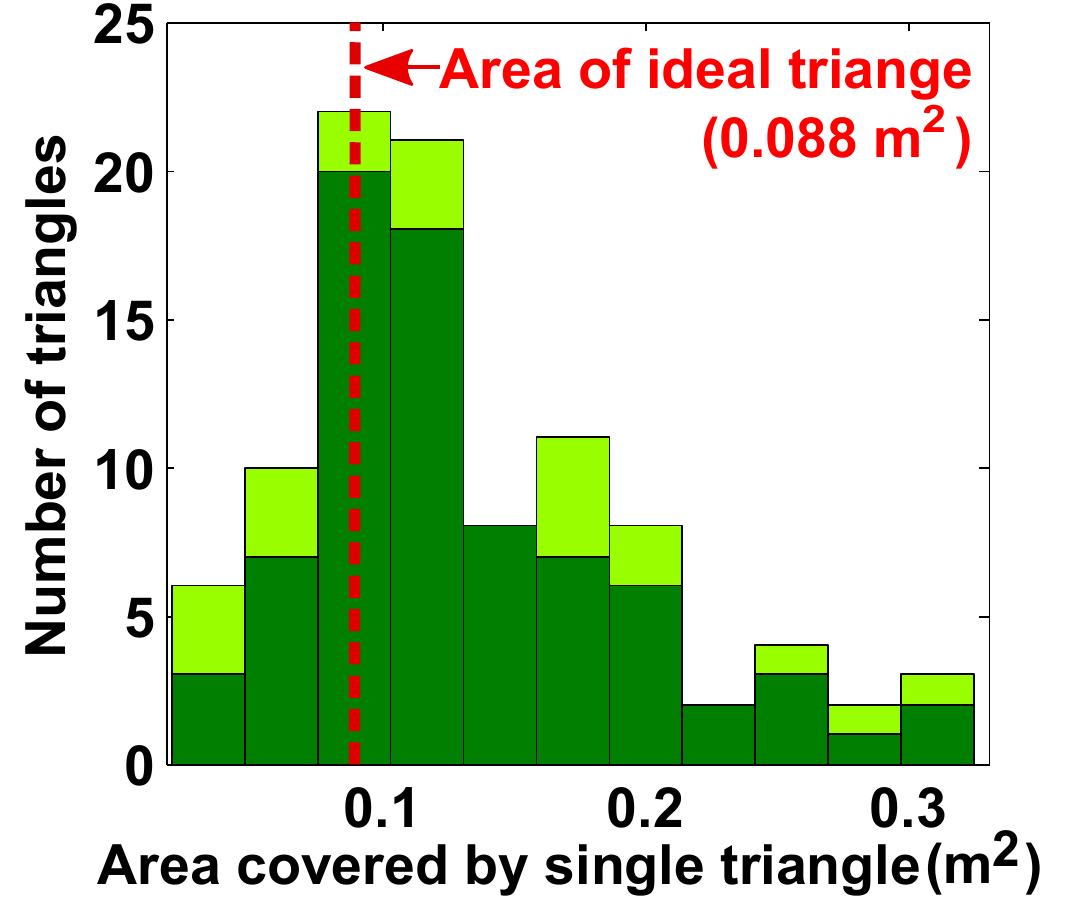}}
  %
\caption{\textbf{\protect\subref{fig:pieChart}} Pie chart for covered area by all triangles. \textbf{\protect\subref{fig:Histogram}} Histogram of covered area by each triangle.}
\end{figure}

Figs.~\ref{fig:minangle} and ~\ref{fig:RatioDelaunay} show our measurements of individual triangle quality; the distribution of minimum angle and maximum/minimum edge length ratio (MaxMin ratio) for each triangle.  The individual data shows triangle quality in a way that overall $\rho$ cannot.  An ideal equilateral triangle has a minimum angle of $\frac{\pi}{3}$ rad and MaxMin ratio of 1. Triangles satisfying the lower bound for minimum angle and the upper bound for MaxMin ratio are 95$\%$ and 96.7$\%$ of overall triangles, respectively.
We note that all triangles not constrained by a wall satisfy these bounds, meaning they are approximately the correct shape, but not always the correct size.  Knowing range would let us address this, but it is unclear how robots expanding the triangulation should choose between making a triangle of the correct shape, or the correct size.  We leave this for future work.

\setlength{\intextsep}{0pt}
\begin{figure}[t]
  \renewcommand{\figheight}{1.6in}
\centering
\subfloat[][Histogram for minimum angle of each triangle.]{
  \label{fig:minangle}
  \includegraphics[width=\linewidth]{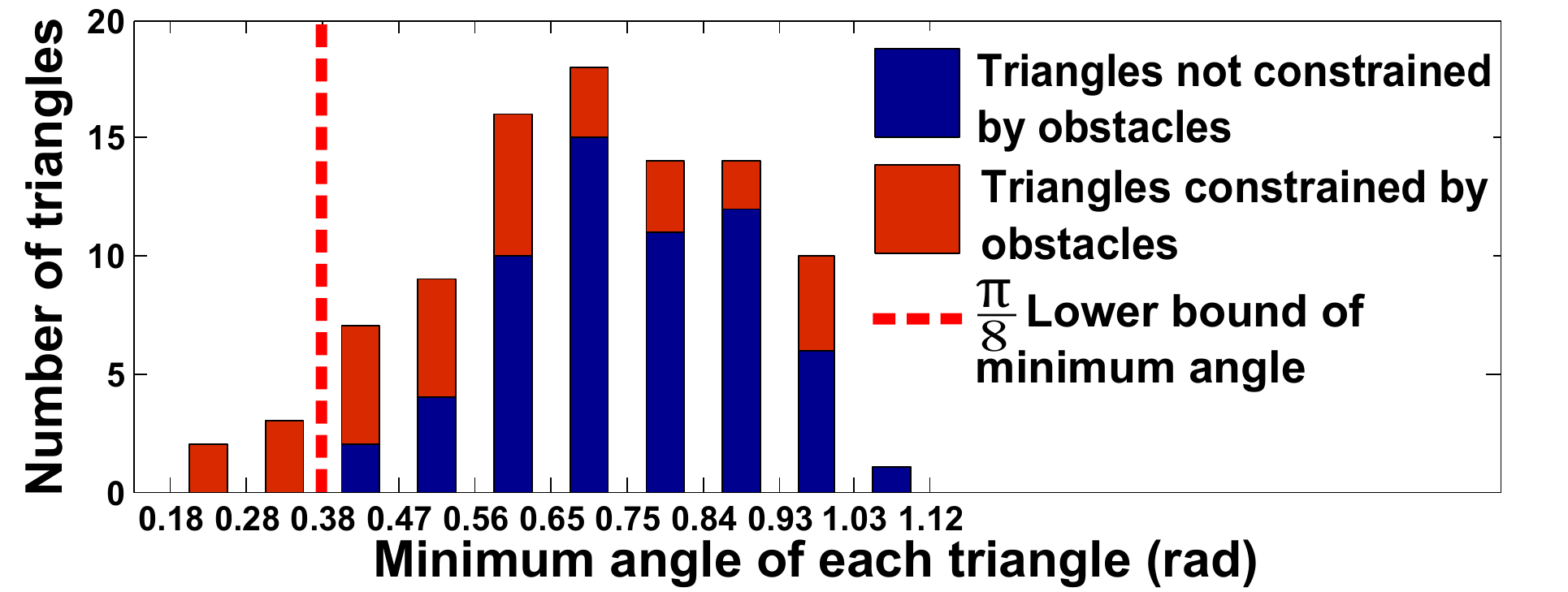}}
\ \
\subfloat[][Histogram for MaxMin ratio of each triangle.]{
  \label{fig:RatioDelaunay}
  \includegraphics[width=\linewidth]{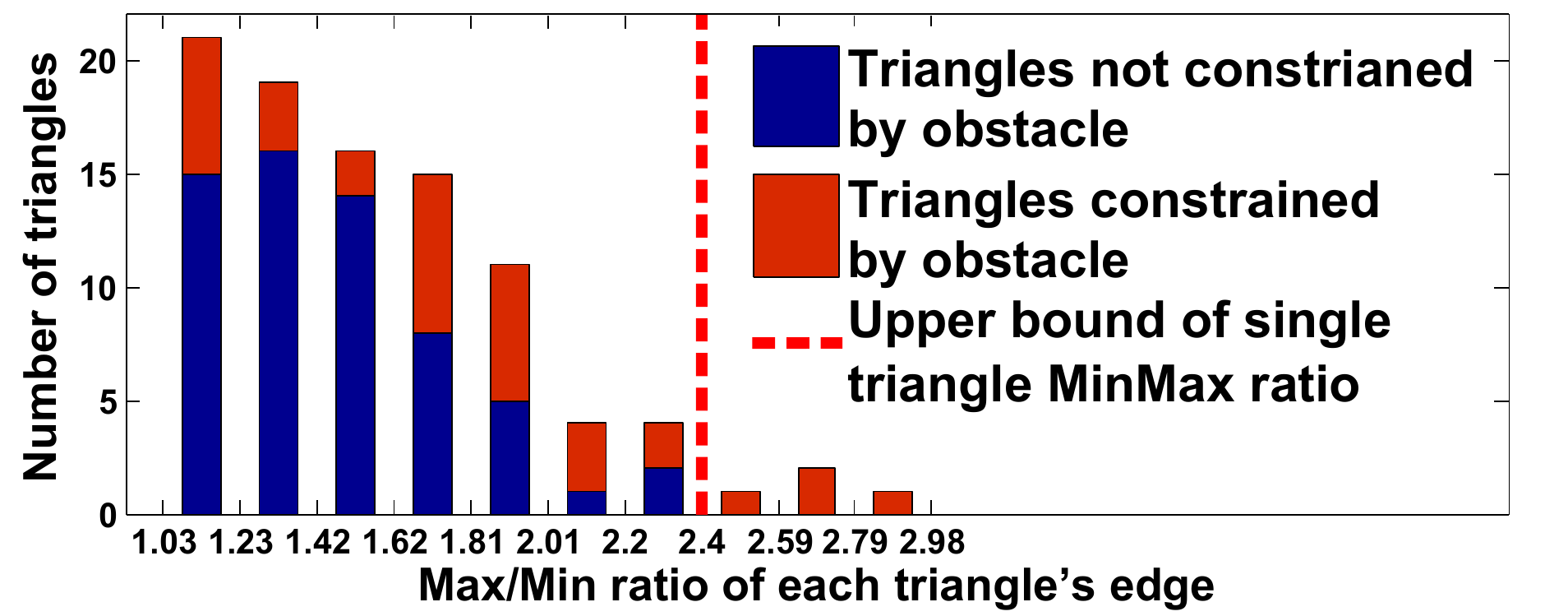}}
\caption{ \textbf{\protect\subref{fig:minangle}} Distribution of minimum angle of each triangle, not for a global $\rho$. All triangles not constrained by a wall satisfy the lower bound, $\frac{\pi}{8}$. 
\textbf{\protect\subref{fig:RatioDelaunay}} Distribution of MaxMin ratio of each triangle. 
All triangles whose minimum angle is larger than the $\frac{\pi}{8}$ (blue colored) also satisfy corresponding upperbound MaxMin ratio. }
\end{figure}


\subsection{Dual Graph Navigation}
%

We start each navigation experiment with a constructed triangulation. The triangulation is $(\rho=1.36, \alpha=0.88rad)${\em -fat}.  For each trial, we randomly select one triangle as a goal, and the robots build a tree on the dual graph. Fig.~\ref{fig:NavPaths}, shows five trials of the 34 we conducted. The numbers inside the triangles indicate the hops in the dual graph from the goal triangle. (The trial from the 8-hop triangle is also shown in Fig.~\ref{fig:NavTestResult})
The thick blue lines show connectivity between owners of adjacent triangles.  Note that this graph is not complete, but it is a spanning graph of all triangle owners in $G$, which is implied by Theorem.~\ref{ConnectedOwners}.

\begin{figure}[t]
\centering
\subfloat[][]{
  \label{fig:NavPaths}
  \includegraphics[width=0.6\linewidth]{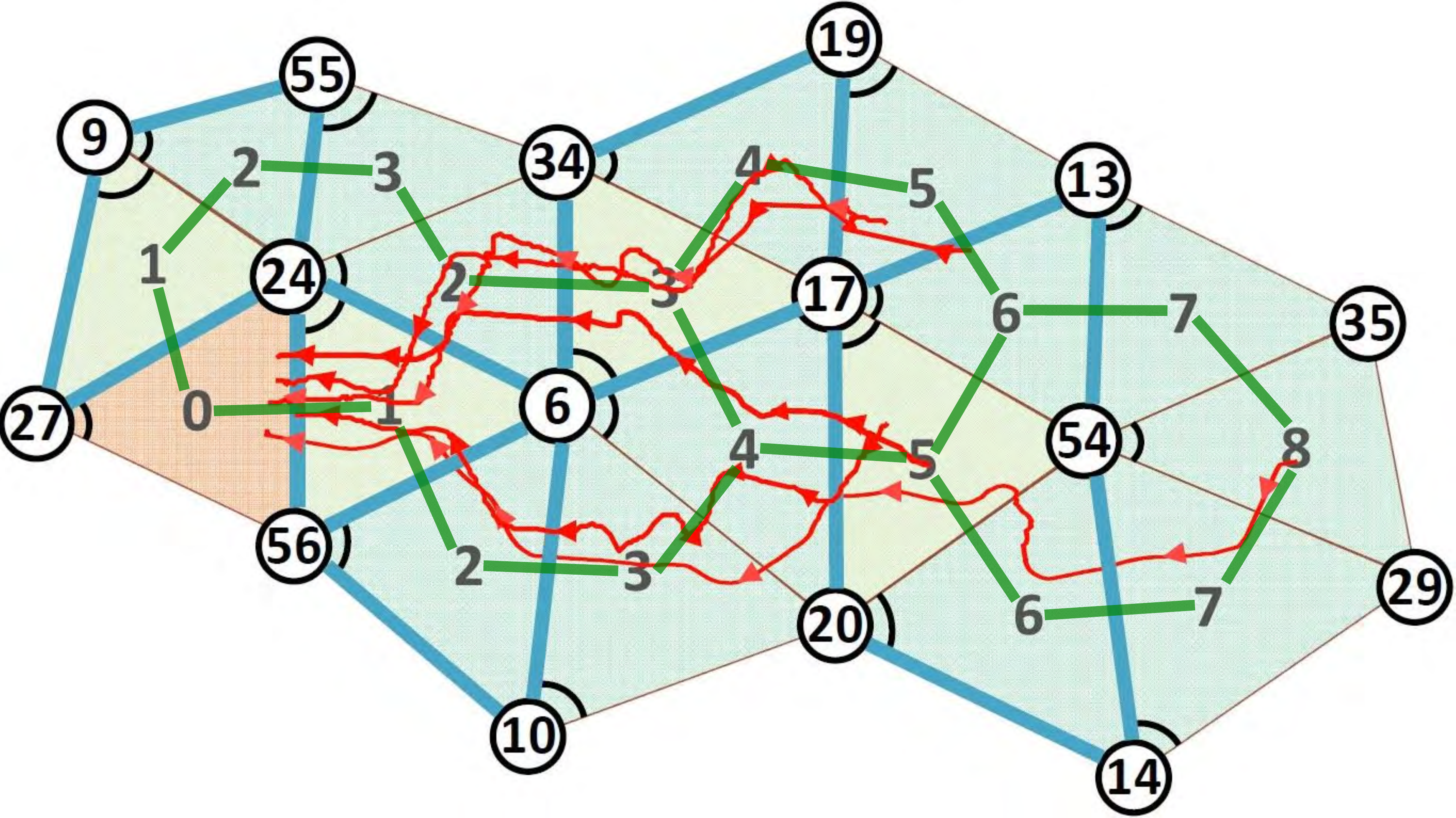}}
\ \
\subfloat[][]{
  \label{fig:SF}
  \includegraphics[width=0.33\linewidth]{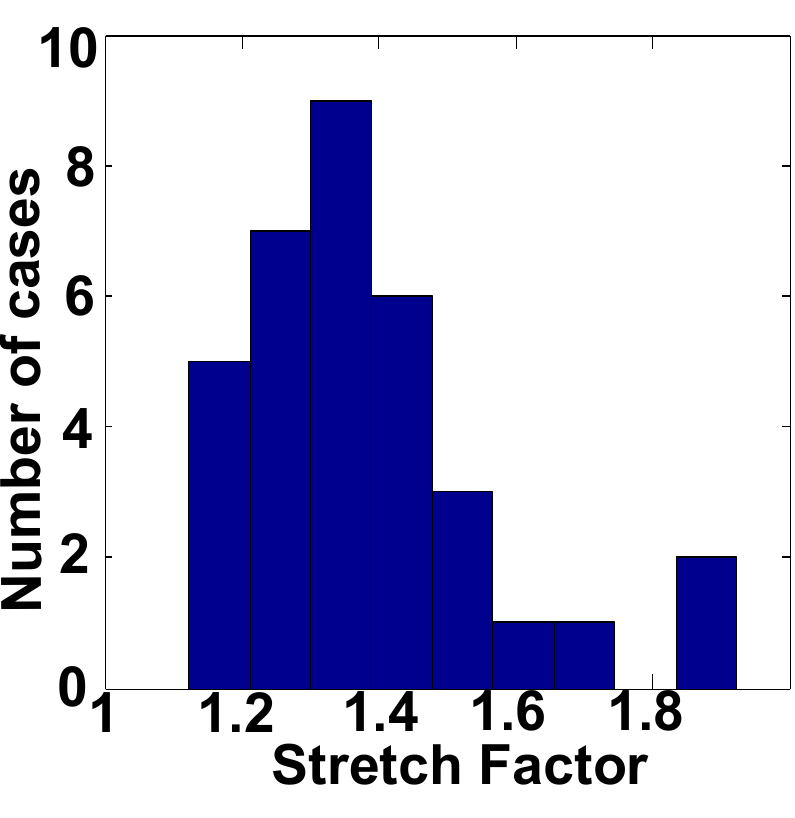}}
 \caption{ \textbf{\protect\subref{fig:NavPaths}} Sample navigation results with r-one robots. \textbf{\protect\subref{fig:SF}} Stretch factor histogram with 34 trials of navigation experiment.}
\end{figure}

%
%

Fig.~\ref{fig:SF} shows the distribution of the stretch factor over all trials of navigation tests with various start-goal pairs. The mean stretch factor is 1.38 $\pm$ 0.19. This is much less than the theoretical bound implied by Theorem~\ref{thm:stretch}, which is based on worst-case assumptions.
Our occupancy algorithm produces 91$\%$ correctness in returning the triangle that actually includes $u$. We define navigation correctness as the ratio of times $u$ moves to the correct adjacent triangle. This result is 99$\%$, with incorrect navigation caused by occupancy errors.

%% file: 06-Conclusion.tex
\section{Conclusion}
\label{sec:Conclusion}
%

We have presented a distributed algorithm to triangulate a workspace, produce a physical data structure, and use this structure for communications and robot navigation.  There are many exciting new challenges that lie ahead. The next step is to extend this approach with a self-stabilizing algorithm that can construct and repair the triangulation and dual graph dynamically, starting from an arbitrary distribution of robots. While existing controllers can already form triangulated graphs~\cite{w._m._spears_distributed_2004}, what is needed is construction and maintenance of the physical data structures, i.e., the primal and dual graphs.  Future work can extend these ideas to very large
populations and dynamic environments. Another objectives will be to improve the routing algorithm to replace the simple dual graph paths by more sophisticated geodesic trajectories.  We are currently working on multi-robot patrolling using the physical data structure to store visitation frequencies and implement a geodesic Lloyds controller
to provide periodic coverage of the triangulation with multi-patrolling robots.